\documentclass[10pt]{article}
\usepackage[accepted]{rlj} 
\usepackage{amssymb}           
\usepackage{mathtools}         
\usepackage{mathrsfs} 
\usepackage{natbib}
\usepackage{graphicx}      
\usepackage{subcaption}         
\usepackage[space]{grffile}     
\usepackage{url}                
\usepackage{lipsum}             
\usepackage{amsmath}
\usepackage{amsthm} 
\usepackage{epsfig}

\usepackage{algorithm, algpseudocode}
\usepackage{thmtools}
\usepackage{outlines}

\usepackage{bbm}

\def\R{\mathbb{R}}

\newtheorem{fact}{Fact}
\newtheorem{lem}{Lemma}

\newtheorem{thm}{Theorem}

\newcommand{\beqan}{\begin{eqnarray*}}
\newcommand{\eeqan}{\end{eqnarray*}}

\newcommand{\beqa}{\begin{eqnarray}}
\newcommand{\eeqa}{\end{eqnarray}}

\def\hist{\mathcal{H}_{t}}

\def\KL{\texttt{KL}}

\def\E{\mathbb{E}}

\def\H{\mathbb{H}}

\newcommand{\A}{\mathcal{A}}
\newcommand{\Prob}{\mathbb{P}}
\newcommand{\argmax}{\mathop{\mathrm{argmax}}}

\newcommand{\Dir}{\texttt{Dir}}
\newcommand{\DP}{\texttt{DP}}

\newcommand{\bH}{{\mathbb{H}}}

\title{Leveraging priors on distribution functions for multi-arm bandits}

\setrunningtitle{Leveraging priors on distribution functions for multi-arm bandits}

\author{Sumit Vashishtha, Odalric-Ambrym Maillard}

\emails{\{sumit.vashishtha, odalric.maillard\}@inria.fr }

\affiliations{
Inria, Univ. Lille, CNRS, Centrale-Lille, UMR 9198-CRIStAL, F-59000 Lille, France
}

\contribution{
     We introduce Dirichlet Process Posterior Sampling (DPPS) for multi arm bandits - a Bayesian nonparametric extension of Thompson sampling based on Dirichlet Processes that combines the strength of (Bayesian) bootstrap with a principled mechanism of {incorporating and exploiting prior information}. 
     }
     {
     Efficient performance of {parametric} Thompson sampling is limited to bandit environments wherein it's possible to have conjugate prior/posterior distributions. Besides, existing Bootstrap based algorithms cannot account for uncertainity that doesn't come from observed data~\citep{osband2018randomized}
     }

\contribution{
  We employ stick-breaking representation of the Dirichlet Process priors to perform numerical experiments with DPPS in both synthetic and real-world multi-arm bandit settings.  
    }
    {
    Improved performance of DPPS compared to parametric Thompson-sampling and UCB is made apparent in these simulations. Using a simple example, we also illustrate a proof-of-concept of the flexibility of DPPS in incorporating prior-knowledge about the bandit environment. Besides, Stick-Breaking implementation of DPPS provides a unified implementation for different bandit environments unlike parametric Thompson sampling  whose implementation differ according to bandit environments and require careful tuning/approximations. 
    }

\contribution{
    We extend the information theoretic analysis of Thompson sampling in~\cite{russo2016information} to a wider class of probability-matching algorithms that derive their posterior probability of optimal action using a valid Bayesian approach, and use this extension to establish $\mathrm{\sigma\sqrt{2TK\log K}}$ non-asymptotic upper bound on the Bayesian regret of DPPS in bandit environments with $\sigma$ sub-Gaussian reward noise, where $\mathrm{T}$ is the time horizon, and $\mathrm{K}$ is the number of arms. 
    }
    {
    We are unaware of any Bootstrap based bandit algorithm that enjoys the order-optimal, $\mathrm{\sigma\sqrt{2TK\log K }}$, non-asymptotic regret bound in the wide class of $\sigma$-sub-Gaussian bandit environments.
    }


\summary{\noindent We introduce Dirichlet Process Posterior Sampling (DPPS), a Bayesian non-parametric algorithm for multi-arm bandits based on Dirichlet Process (DP) priors. Like Thompson-sampling, DPPS is a probability-matching algorithm, i.e., it plays an arm based on its posterior-probability of being optimal. Instead of assuming a
parametric class for the reward generating distribution of each arm, and then putting a prior on the  parameters, in DPPS the reward generating distribution is directly modeled 
using DP priors. DPPS provides a principled approach to incorporate prior belief about the bandit environment, and in the noninformative limit of the DP priors (i.e. Bayesian Bootstrap), we recover Non Parametric Thompson Sampling (NPTS), a popular non-parametric bandit algorithm,  as a special case of DPPS. We employ stick-breaking representation of the DP priors, and show excellent empirical performance of DPPS in challenging synthetic and real world bandit environments. Finally,  using an information-theoretic analysis, we show non-asymptotic optimality of DPPS in the Bayesian regret setup. 

}

\begin{document}

\makeCover  
\maketitle  

\begin{abstract}
 We introduce Dirichlet Process Posterior Sampling (DPPS), a Bayesian non-parametric algorithm for multi-arm bandits based on Dirichlet Process (DP) priors. Like Thompson-sampling, DPPS is a probability-matching algorithm, i.e., it plays an arm based on its posterior-probability of being optimal. Instead of assuming a
parametric class for the reward generating distribution of each arm, and then putting a prior on the  parameters, in DPPS the reward generating distribution is directly modeled 
using DP priors. DPPS provides a principled approach to incorporate prior belief about the bandit environment, and in the noninformative limit of the DP posteriors (i.e. Bayesian Bootstrap), we recover Non Parametric Thompson Sampling (NPTS), a popular non-parametric bandit algorithm,  as a special case of DPPS. We employ stick-breaking representation of the DP priors, and show excellent empirical performance of DPPS in challenging synthetic and real world bandit environments. Finally,  using an information-theoretic analysis, we show non-asymptotic optimality of DPPS in the Bayesian regret setup. 

\end{abstract}

\section{Introduction}

Multi Arm Bandits (MAB) is a paradigmatic framework to study the exploration $\sim$ exploitation dilemma in sequential decision making under uncertainty. Standard algorithms developed within this framework such as Upper-Confidence Bounds (UCB) based algorithms~\citep{auer2002finite} and Thompson sampling (TS)~\citep{thompson1933likelihood,russo2018tutorial} have proven to be useful in applications such as clinical trials, ad-placement strategies, etc. However, it remains difficult to apply them to more complicated real world settings such as those arising in agriculture or experimental sciences wherein the underlying uncertainty mechanism is far more sophisticated: the unknown reward distribution corresponding to each arm/action may not even conform to a parametric class of distributions such as the single-parameter exponential family, and usually exhibit characteristics such as multi-modality.  With some abuse of terminology, we shall refer to this challenging setting of the MABs as \textit{non-parametric} MABs, and we report an optimal algorithm for this setting in the current paper.  

To begin with, it’s worthwhile to consider the limitations of UCB and Thompson sampling algorithms in some more detail.  Firstly, the  efficient performance of UCB type algorithms rely on the construction of tight high-probability confidence sequences~\citep{abbasi2011regret,auer2002finite}. However, for complex problems, it becomes difficult to design such sequences, and only approximate confidence sequences can be designed, which generally tend to be statistically suboptimal~\citep{filippi2010parametric}. Next, although  Thompson-Sampling (TS) ~\citep{thompson1933likelihood,kaufmann2012thompson} is a neat and elegant \textit{Bayesian} algorithm, that enjoys the flexibility of incorporating \textit{prior} knowledge about the bandit environment, it's efficiency is limited to the regime of \textit{conjugate} prior/posterior distributions of the relevant scalar/vector parameter, which is generally not possible beyond a few special cases of bandit environments, e.g. Bernoulli, Gaussian. In other regimes, the posterior distributions no longer exhibit a closed form, and require the application of approximate inference schemes such as Markov Chain Monte Carlo (MCMC), variational inference, etc, to draw samples from the posterior distributions.  This is usually computationally expensive, and can easily lead to the suboptimal performance of Thompson sampling~\citep{phan2019thompson}.

In light of the above limitations, one is tempted to look for a statistical-inference technique suitable for handling complicated real-world distribution functions, and soon finds the answer in Statistical Bootstrap ~\citep{efron1992bootstrap,efron1994introduction} which is a procedure for estimating the distribution of an estimator by resampling (often with replacement) data or a model estimated from the data. Bootstrapping has been widely used as an alternative to statistical-inference based on the assumption of a parametric-model when that assumption is in doubt, or where parametric inference is impossible or requires complicated formulas for the calculation of standard errors. 

This naturally motivates the use of statistical-Bootstrap for the nonparametric setting of MAB discussed above. In fact, most of the existing algorithms for nonparametric MABs are based on different versions of the Bootstrap in one way or the other~\citep{kveton2019garbage,baransi2014sub,osband2015bootstrapped}. However, these methods crucially rely on \textit{artifical history/pseudo-rewards} to perform well, and can perform sub-optimally without a suitable mechanism to generate such artificial-history/pseudo-rewards~\citep{osband2015bootstrapped}. Additionally,  these bootstrap sampling based algorithms cannot account for uncertainty that does not come from the observed data~\citep{osband2016deep}. In other words, they do not have a mechanism to incorporate \textit{prior} knowledge about the environment which can be utilized to enhance the performance of the algorithm. This efficient harnessing of prior knowledge for improved performance is hallmark of Bayesian algorithms, and we are unaware of any bandit algorithm that enjoys the flexibility of being completely Bayesian and still efficient in the nonparametric MAB setting. Essentially, this calls for an extension of the parametric Thompson sampling, which is already Bayesian, but suffers its nemesis in the non-parametric MAB setting for reasons discussed before.  Consequentially, this leads us to the following question,

\textit{Can we design a truly Bayesian algorithm that performs efficiently in the setting of nonparametric multi-arm bandits?}

We answer this question in the affirmative by designing an algorithm that draws from the strengths of Bayesian Nonparametric (BN) priors. In the past, a nice line of work utilized BN priors on the \textit{function spaces}, i.e. Gaussian Process (GP) priors, to contribute the well known GP-UCB algorithm~\citep{srinivas2009gaussian}, but it's not clear how this can be naturally adapted to the nonparametric MAB setting that we are interested in the current paper, and we believe that a more natural choice of BN priors in the context of multi-arm bandits would be the priors on the space of probability distributions instead of those on a much larger function space (restricted only by the choice of their smoothness)~\citep{rasmussen2003gaussian}. Dirichlet Processes (DPs), denoted as $\mathrm{DP}(\alpha, \mathrm{F_{0}})$, (where $\alpha$ and $\mathrm{F_{0}}$ are the related hyperparameters, known as the concentration parameter, and the base measure respectively), fall in the category of BN priors on the space of probability distributions, and have been widely used in real world statistical applications~\citep{castillo2024bayesian,muller2015bayesian,ghosal2010dirichlet}, . We extend the strength of DPs to the multi-arm bandit setting by contributing Dirichlet Process Posterior sampling (DPPS). 

DPPS directly treats reward distribution functions as \textit{random objects}, modeling them using DP priors, and easily updating these priors utilizing the property of {conjugacy} of DP priors to obtain DP posteriors, and making decisions based on the the posterior probability of optimal actions induced by these DP posteriors. Since no parametric class of distribution for the arm reward distributions is assumed apriori, DPPS allows for modeling arbitrary reward distributions, and hence  is amenable to the non-parametric MAB setting.   This is in contrast to parametric Thompson sampling which assumes a parametric class for reward distribution apriori, and puts a prior on a scalar/vector parameter, often the sufficient-statistic of that parametric-class, thereby restricting its application to a small set of problems. Furthermore, these parametric priors do not enjoy the property of conjugacy very often, and it becomes challenging to sample from their posterior distributions even for the restricted class of problems they can model appropriately.  We will illustrate this strength of DPPS in a series of numerical experiments in Section~\ref{sec:Numerics} for different bandit environments.

Since DPPS is a Bayesian algorithm, it provides a principled mechanism to incorporate prior knowledge about the bandit environment, specifically through the base measure of the DP priors. In fact, based on the hyperparameter, $\alpha$, of the DP prior it's easy to delineate uncertainty captured in DP priors/posteriors into two parts --  contributions from the observed data and contributions from the prior. In the limit of $\alpha \to  0$, one recovers the noninformative DP prior, also referred to as \textit{Bayesian Bootstrap} which is the basis for Non Paramteric Thompson sampling introduced in \cite{riou2020bandit}.    We discuss this in Section ~\ref{subsec:NPTS-DPPS}, and also give a proof of concept of the flexibility of DPPS to incorporate prior knowledge about bandit environment through a simple example in Section~\ref{sec:Numerics}. Additionally, in Section~\ref{sec:Theoretcial-analysis}, we extend an elegant information-theoretic analysis framework for parametric Thompson sampling to a wider set of probability matching algorithms that derive the posterior probability of optimal actions using a valid/proper Bayesian strategy. This extension, along with an important lemma on the tail of random distributions sampled from DP prior/posterior shall lead us to the result of Theorem~\ref{thm:Bayesian-regret-DPPS} which provides an upper bound on the Bayesian regret of DPPS. 

\section{Problem formulation }
\label{sec:Setup-notation}

In this section, we formalize the problem of multi-arm bandits and introduce the necessary notation. We also discuss Thompson-sampling, a Bayesian  probability matching algorithm, in order to lay some ground for introducing its nonparametric counterpart, DPPS, later in this paper. 

\paragraph{Multi-armed bandits} In the $K$-arm bandit problem, the agent is presented with $K$ arms/distributions/actions $\{p_k\}_{k=1}^K$. At time-steps $t=0,1,\ldots$, the agent executes an action $A_t \in \mathcal{A}$, $\mathcal{A}$ being the set of actions such that $|\mathcal{A}| = K$; then it observes the corresponding reward $R_{t,A_t}\in \raisebox{2pt}{$\chi$}$. In this paper, we choose $\raisebox{2pt}{$\chi$}$ to the set of $\sigma$-sub-Gaussian random variables, i.e. $\E\left[e^{(X -\E[X])t}\right] \le e^{\frac{\sigma^{2}t^{2}}{2}}$, $\forall X \in \raisebox{2pt}{$\chi$}$, and for all $s$.  Let $R_{t} \equiv (R_{t,a})_{a\in\mathcal{A}}$ be the vector of rewards at time $t$. The “true reward-vector distribution” $p^\star$ is seen as a distribution over  $\raisebox{2pt}{$\chi$}^{|\mathcal{A}|}$ that is itself randomly drawn from the family of distributions $\mathcal{P}$. We assume that, conditioned on $p^\star$, $(R_{t})_{t\in N}$ is an iid sequence with each element $R_{t}$ distributed according to $p^\star$.  The agent's experience through time-step $t$ is encoded by a history $\hist = (A_1,R_{1,A_1},\ldots,A_{t},R_{t,A_{t}})$. The action $A_{t}$ is chosen based on $\hist$ utilizing a sequence of deterministic functions, $\pi = (\pi_{t})_{t\in N}$, so that  $\pi_{t}(a) = \Prob(A_{t}=a|\hist)$. $\pi$ is usually referred to as randomized \textit{ policy}.  The $T$ period \textit{regret} of the  sequence of actions, $A_1,..,A_T$, induced by $\pi$, is the random variable, 
\begin{equation*}
\label{eq:Byesian-regret}
\mathrm{Regret}(T,\pi) = \sum_{t=1}^{T} \E[R_{t,A^\star} - R_{t,A_t}]
\end{equation*}

where $A^\star \in \mathcal{A}$ is the optimal action, i.e. $A^\star  \in \underset{a \in \A} \argmax \,  \E[R_{1,a}|p^\star]$ . In this paper, we study the expected regret or \textit{Bayesian regret} given as follows,
\[\E\left[\mathrm{Regret}(T,\pi)\right] = \E\left[\sum_{t=1}^{T} \left[R_{t,A^\star} - R_{t,A_t} \right]\right] ,\]

where the expectation integrates over random reward realizations, the prior distribution of $p^\star$, and algorithmic randomness.

\paragraph{Further notation} We set $\alpha_t(a) = \Prob\left( A^\star=a \vert \hist  \right)$ to be the posterior distribution of $A^\star$. Also, we use the shorthand notation $\E_{t}[\cdot]=\E_{t}[\cdot | \hist]$ for conditional expectations under the posterior distribution, and similarly write $\Prob_{t}(\cdot)=\Prob(\cdot | \hist)$. 
For two probability measures $P$ and $Q$ over a common measurable space, if $P$ is absolutely continuous with respect to $Q$, the \textit{Kullback-Leibler divergence} between $P$ and $Q$ is
\begin{equation}
\KL(P || Q)= \int P \log \left( \frac{dP}{dQ} \right)dP
\end{equation}
where $\frac{dp}{dq}$ is the Radon--Nikodym derivative of $p$ with respect to $q$. For a probability distribution $p$ over a finite set $\mathcal{X}$, the {\it Shannon entropy} of $p$ is defined as $\bH(p) =- \sum_{x\in\mathcal{X}} p(x) \log\left( p(x) \right)$.
The {\it mutual information} under the posterior distribution between two random variables $X_1:\Omega \rightarrow \mathcal{X}_{1}$, and $X_2:\Omega \rightarrow \mathcal{X}_{2}$, denoted by
\begin{equation}
I_{t}(X_1 ; X_2) := \KL\left( \mathbb{P}\left( (X_{1}, X_{2})  \in \cdot \vert \hist \right) \,\, || \,\, \mathbb{P}\left( X_{1}  \in \cdot \vert \hist \right) \mathbb{P}\left( X_{2}  \in \cdot \vert \hist \right)  \right),
\end{equation}
is the Kullback-Leibler divergence between the joint posterior distribution of $X_{1}$ and $X_{2}$ and the product of the marginal distributions. Note that $I_{t}(X_1 ; X_2)$ is a random variable because of its dependence on the conditional probability measure $\Prob\left( \cdot \vert \hist \right)$.

\paragraph{Thompson Sampling}

Thompson Sampling is a specific class of probability matching algorithms which \textit{matches} in each round, the action-selection probability  to the posterior probability-distribution of optimal action, i.e. $\Prob(A_{t} = a|\mathcal{H}_{t})=\Prob(A^\star=a|\mathcal{H}_{t})$.  First, a parametric class for the reward distribution functions $\{\pi_k\}_{k=1}^K$ is assumed, such that for each arm there is a $\theta_{a}$ which maps the arm to a distribution in that class.  Thompson sampling is a Bayesian algorithm in the sense that it considers each of these unknown $\theta_{a}$, as a random variable initially distributed according to a prior distribution, i.e.,  $\theta_a \sim \pi_{a,0}$, and this prior evolves to a  posterior distribution, $\pi_{a,t}$, in round $t$, through Bayes rule, as rewards are obtained in each round. At each time, a sample $\theta_{a,t}$ is drawn from each posterior $\pi_{a,t}$, and then the algorithm chooses to sample  $a_t=\arg \max_{a\in\{1,\dots,K\}}\{\mu(\theta_{a,t})\}$, where $\mu(\theta_{a,t})$ represents the mean of the  parametric reward distributions with parameter $\theta_{a,t}$. 






\section{Background on Dirichlet processes}
\label{sec:Background-DP}

Before discussing the main algorithm proposed in this paper, It is important to concretely discuss a few key aspects concerning Dirichlet Processes, and this is what we do in this section. 

\paragraph{Dirichlet distribution} is a multivariate generalization of the Beta distributions.  We denote the Dirichlet distribution of parameters $(\alpha_{1},. . .,\alpha_{n})$ by $\Dir(\alpha_{1},. . .,\alpha_{n})$ whose density function is given by $\frac{\Gamma(\sum_{i=1}^{n}\alpha^{i})}{\prod_{i=1}^{n}\Gamma(\alpha^{i})} \prod_{i=1}^{n}w_{i}^{\alpha^{i}-1}$ for $(w_{1},. . .,w_{n}) \in [0,1]^{n}$ such that $\sum_{i=1}^{n}w_{i} = 1$
\paragraph{Dirichlet Processes} In the Bayesian formalism (see also section~\ref{sec:Bayesian-framework} for more details), an unknown object is treated as a random variable which is then assumed to be drawn from a prior distribution. A Bayesian solution requires developing methods of computation of the posterior distribution from this prior based on available information about the unknown object. When the unknown object is a probability measure (a cumulative distribution function in the present paper, to be precise), one then faces a non-trivial question of how to even define a prior on an infinite dimensional object and also take care of the constraints of a probability measure (sum up to 1 over its support).  An elegant solution was offered in \cite{ferguson1973bayesian} wherein the author introduced the idea
of a Dirichlet process (DP) – a probability distribution on the space of probability measures which induces finite-dimensional Dirichlet distributions when the data are grouped. To look at it concretely, consider a  random probability measure, $G$, on some nice (e.g. Polish) space $\Theta$ (e.g. $\R$). $G$ is said to be DP distributed with base probability measure $F$ (e.g. a Gaussian, Beta, Bernoulli, etc) and concentration parameter $\alpha \in \mathbb{R}^{+}$, denoted as $G \sim \DP(\alpha,F)$, if
\[(G(A_1), . . . , G(A_r)) \sim \Dir(\alpha F(A_1), . . . , \alpha F(A_r)) \]
for every finite measurable partition $A_1,...,A_r$ of the space $\Theta$. 

Having witnessed the construction of DP priors on the space of probability measures, one naturally wonders, how to derive posteriors from these priors, and for that we discuss the important property of \textit{conjugacy} in some nonparametric priors.

\paragraph{Conjugacy} In the Bayesian parametric framework, one can usually use Bayes rule for deriving posteriors for parametric models, however for non-parametric case, Bayes rule cannot be used in general (see Appendix ~\ref{subsec:Bayes_rule} for technical details). Posteriors for some nonparametric priors can be derived utilizing the property of conjugacy. Particularly, an observation model $M \in \mathcal{G}$, and the family of priors $\mathcal{Q}$ are called conjugate if, for any sample size $n$ and any observation sequence $X_{1}, . . ., X_{n}$, the posterior under any prior $Q\in \mathcal{Q}$ is again an element of $\mathcal{Q}$. Also, merely possessing the property of conjugacy is not enough to form a viable Bayesian prior.  For example, a generalization of DPs is the so-called Neutral To The Right (NTTR) processes~\citep{dey2003some}. Entire family of NTTR is  known to be conjugate, but besides the specific case of DPs, there's no known explicit method of obtaining \textit{posterior indices} in other members of the NTTR family. This leads us to discuss the form of DP posteriors next.

\paragraph{Dirichlet Process posteriors} Let $X_1$, ... , $X_n$ be a sample from an unknown real-valued distribution $G_{0}$ where $X_i \in \R$. To estimate $G_{0}$ from a Bayesian perspective (see Appendix~\ref{sec:Bayesian-framework} ) we put a prior on the set of all distributions $\mathcal{G}$ and then we compute the posterior distribution of $G_{0}$, given ${\bf X}_n = (X_1, . . . , X_n)$. Let's put a DP prior   on the set $\mathcal{G}$. Correspondingly, Let $\DP(\alpha,F_{0})$, denote the DP prior. The distribution $F_{0}$ can be thought of as a prior guess at the true distribution $G_{0}$. The
number $\alpha$ controls how tightly concentrated the prior is around $F_0$.  With a DP prior on $G_{0}$, the posterior of $G_{0}$, given ${\bf X}_n = (X_1, . . . , X_n)$, enjoys \textit{conjugacy}, i.e, it is itself a DP given as, $\DP(\alpha_{n},\overline{F}_n)$, where, the \textit{posterior indices}, $\alpha_{n}$, and $\overline{F}_{n}$ are obtained as follows~\citep{ferguson1973bayesian,ghosal2010dirichlet},
\begin{align} \label{def:DP-posteriors}
\alpha_{n} = \alpha + n, \, \overline{F}_n = \frac{n}{\alpha+n}{F}_n + \frac{\alpha}{\alpha+n}F_0 
\end{align}
\noindent Here ${F}_n$ is the \textit{empirical  distribution function} given $X_1,. . .,X_{n}$, i.e., $F_{n}(x) = \frac{1}{n}\sum_{i=1}^{n}\mathbb{I}(X_{i} \le x)$. 

Note how the posterior index, $\overline{F}_n$,  exhibited in Eq.~\ref{def:DP-posteriors} combines the information from observations (via the empirical cdf, $F_{n}(x)$ ) with that available from the prior (using $F_{0}$). This is a crucial property of DPs that our algorithm , DPPS, shall harness in order to account for information obtained via  observed data, and the prior information.  One can easily see that as $\alpha \to 0$, DPs can only account for uncertainty obtained via observations, with no role of prior anymore, and we discuss this next. 
\paragraph{Bayesian Bootstrap}  A very useful idea in statistical inference has been that of Statistical Bootstrap~\citep{efron1994introduction}, and a Bayesian version of Bootstrap was introduced in ~\citep{rubin1981bayesian}. Interestingly, this Bayesian version of Bootstrap can also be derived as a special case of the DP posteriors~\citep{ghosal2017fundamentals}. Specifically, the weak limit, $\DP(n,\sum_{i=1}^{n}\delta_{X_{i}})$, (also referred to as the \textit{noninformed limit} sometimes) of the DP posterior, $\DP(\alpha_{n},\overline{F}_n)$, as $|\alpha|\to 0$ is known as Bayesian Bootstrap (BB), and is given as,
\begin{align}\label{def:Bayesian-Bootstrap}
\mathrm{BB_{n}}:= \DP(n,\sum_{i=1}^{n}\delta_{X_{i}}) = \sum_{i=1}^{n} W_{i}\delta_{X_{i}}
\end{align}
where ${\bf W}_n=(W_{1},...,W_{n}) \sim \Dir (1,...,1)$, and $X_{i}$ are the observed data points. The mean of a random distribution drawn from Bayesian-Bootstrap can be easily seen to be the dot-product between the weights and the observed data-points, i.e.,
\begin{align}\label{def: BB-mean}
\mu(BB_{n}) = \sum_{i=1}^{n}W_{i}X_{i} = \left<{\bf W}_n,{\bf X}_n\right>
\end{align}
As we shall see in Sec~\ref{sec:DPPS}, the idea of Bayesian Bootstrap forms the basis for a  bandit algorithm introduced in~\citep{riou2020bandit}. Next we discuss an important representation of DP priors/posteriors that make them amenable to practical applications.

\paragraph{Stick-breaking representation of DPs} With the necessary details about DP prior and posterior distributions set, one naturally asks how to draw sample from these distributions because this is necessary if one wants to do any sort of statistical inference using DPs. Particularly, the form of DP posterior (indices) in Eq.\ref{def:DP-posteriors} provide little information to answer this question. A representation of random measures sampled from DPs, reported in ~\citep{sethuraman1994constructive}, known as Stick Breaking representation of DPs, provides an answer to this question. 
In general, Stick-breaking measures~\citep{ishwaran2001gibbs} are almost surely discrete random probability measures that can be represented as,
\begin{align}\label{def:Stick-Breaking-priors}
{Q}(\cdot) = \sum_{i=1}^{N}q_{i}\delta_{Z_{i}}(\cdot)
\end{align}
where $\delta_{Z_{i}}$ is a discrete measure concentrated at $Z_{i}$, and $q_{i}$  are random weights, generated independent of $Z_{i}$, such that $q_{i}\in [0,1]$, and $\sum_{i=1}^{N}q_{i} = 1$. As one can guess, this is analogous to breaking an actual stick into pieces, and hence the name. In a seminal paper,  \cite{sethuraman1994constructive} reported  that if these weights, $q_{i}$, are constructed  such that,
\begin{align} 
\label{def:SB-DP}
q_{1}&= V_{1}, \: (q_{i})_{i=2}^{N-1}= V_{i}\prod_{j=1}^{i-1}(1-V_{j}), \: q_{N}= \prod_{i=1}^{N}(1-V_{i})\\ \label{def:SB-DP-3}
V_{i}& \stackrel{iid}{\sim}\texttt{Beta}(1,\alpha), \: Z_{i} \stackrel{iid}{\sim} F,\: i=1,2,...N
\end{align} 
and $N$ is $\infty$, then the generated random discrete measure, $P$, in Eq.\ref{def:SB-DP} (with $N$ as $\infty$) is such that, $P \sim \DP(\alpha,F)$. Ofcourse, for computation one can't have $N$ as $\infty$, and the infinite series is truncated at some finite $N$, such that a probability mass, $q_{N} = 1 - \sum_{i=1}^{N-1} q_{i} = \prod_{i=1}^{N}(1-V_{i})$, is put at the last point, $Z_{N}$, and this construction ensures that all weights, $q_{i}$ sum up to one.  This finite Stick-breaking representation has been widely used~\citep{ishwaran2001gibbs,muliere1998approximating} thanks to its provable optimality in closely approximating the infinite series (see also Appendix Section~\ref{sec:stick-breaking-finite} for this and for more details on choosing finite $N$, etc).

\paragraph{Iterative form of DP posterior} With the stick-breaking representation of DP priors at hand, one wonders how to sample from DP posteriors (i.e. $\DP(\alpha_{n},\overline{F}_n)$) in a practically feasible way and, for this, an iterative form of DP posterior~\citep{blackwell1973ferguson,sethuraman1994constructive} comes in handy, and is given as follows,
\begin{align}\label{eq:single-step-posterior}
Q_{i}(\cdot) \stackrel{d}{=} V_{i}\delta_{X_{i-1}} + (1-V_{i})Q_{i-1}(\cdot)
\end{align}
Here $V_{i} \sim \texttt{Beta}(1,\alpha + i)$, and $\stackrel{d}{=}$ denotes equality in distribution. Beginning with a random measure sampled from the DP prior, $Q_{0}$, generated using the stick-breaking method (Eqs.\ref{def:SB-DP}-\ref{def:SB-DP-3}),
the recursion in Eq.\ref{eq:single-step-posterior} is simply applied $N$ times (corresponding to $N$ observations $\{X_{1},. . ., X_{N}\}$) to obtain a random measure, $Q_{N}$, sampled from the DP posterior (i.e. $\DP(\alpha_{n},\overline{F}_n)$), 
\begin{align}
\label{eq:multi-step-posterior}
Q_{N} \stackrel{d}{=} V_{N}\delta_{X_{N}} + \sum_{i=1}^{N-1}\left[V_{i}\prod_{j=i+1}^{N}(1-V_{j})\right]\delta_{X_{i}}  + \left[\prod_{i=1}^{N}(1-V_{i})\right] Q_{0}.
\end{align}

\section{Dirichlet process posterior sampling} \label{sec:DPPS}

Having established the necessary background, we are now ready to introduce our algorithm, DPPS.  

Algorithm~\ref{alg:DPPS} gives the pseudo-code for DPPS. Instead of assuming a \emph{parametric} class for the reward generating distribution of each arm, and then putting a prior on the parameter, we model the reward generating distribution of each of the arms $\{p_{k}\}_{k=1}^{K}$ using a corresponding DP. In each round, DPPS operates as follows:  a random distribution, $D_{k}$, is sampled from the current DP posterior for each of the $K$ arms utilizing the stick-breaking representation of the DP posterior of Eq.~\ref{eq:multi-step-posterior}; To select an arm, the probability matching principle is followed, that is, the arm with the highest probability of being optimal (i.e. one corresponding to the highest  of the means, $\mu(D_{k})$, of the random measures, $D_{k}$) in that round is pulled. It is denoted as $I(t)$. After observing the reward $R_{t,I(t)}$, 
the history of observed rewards, ${\bf R}_{I(t)}$, for this arm is updated, and the DP posterior of the pulled arm is updated  using the $N_{I(t)}$ observations.  Clearly, DPPS can be seen as Thompson sampling wherein the prior/posterior are nonparametric, instead of parametric\footnote{Note that DPPS is a (non-parametric) Bayesian algorithm that utilizes probability-matching principle for arm selection, and hence is in \textit{exact} sense, Thompson sampling.}. 
As a result, most of the theoretical guarantees and proof techniques for Thompson-sampling apply to DPPS as well. 
An important practical advantage of DPPS is that one does not need to know the parametric-class of distribution functions. More crucially,  the posteriors in parametric Thompson-sampling are often not available in exact form, and must be approximated using expensive inference techniques. This issue does not arise in DPPS, as the resulting posteriors in DPPS are always DP, and one can sample from DP posteriors utilizing their stick-breaking representation discussed in Section~\ref{sec:Background-DP}. Also,  DPPS enjoys the same flexibility as that of DP posteriors  in utilizing information obtained from the observed data and that from some prior knowledge. In other words it  combines the (data-driven) strength of vanilla (Bayesian) Bootstrapping with  the flexibility of incorporating prior beliefs.  

\begin{algorithm}[h]
\caption{Dirichlet Process Posterior Sampling}\label{alg:DPPS}
\begin{algorithmic}[1]
\Require Horizon $T$, number of arms $K$, arm parameters -- Distribution $F_{0,k}$, constant $\alpha_{0,k}$ for $k \in \{1,. . .,K\}$

\For{$k=1. . .K$,}
\State Set ${\bf R}_{k}=\left[\;\right]$, $F_{k}=F_{0,k}$, $\alpha_{k}=\alpha_{0,k}$, and $N_{k}=0$
\EndFor

\For{$t = 1...T$,}
\State \textcolor{blue}{\textit{\#} Sample model (a random measure):}
\For{$k=1. . .K$,}
\State Sample $D_{k}\sim \DP(\alpha_{k},{F}_{k})$
\EndFor
\State \textcolor{blue}{\textit{\#} select and apply action:}

\State $I(t) = \argmax_{k\in\{1,...,K\}}\{\mu(D_{k})\}$
\State Pull arm $I(t)$ and observe reward $R_{t,I(t)}$
\State Update history ${\bf R}_{I(t)} = ({\bf R}_{I(t)}^{\top}, R_{t,I(t)})^{\top}$ and count $N_{I(t)} \leftarrow N_{I(t)} + 1$.
\State \textcolor{blue}{\textit{\#} Posterior update}

\State $\alpha_{I(t)} \leftarrow \alpha_{I(t)} + 1$
\State $F_{I(t)} = \frac{1}{\alpha_{I(t)}}\sum_{x\in{\bf R}_I(t)} \delta_{x} + \frac{\alpha_{0,I(t)}}{\alpha_{I(t)}}F_{0,I(t)}$
\EndFor
\end{algorithmic}
\end{algorithm}

\subsection{Noninformative limit of the DPPS}
\label{subsec:NPTS-DPPS}

In \cite{riou2020bandit}, authors introduced a non-parametric algorithm for multi-arm bandits, calling it Non-Parametric Thompson Sampling (NPTS), although noting that NPTS is not a Bayesian algorithm, and that it is not Thompson sampling in \textit{strict} sense. They proved its asymptotic optimality, and showed empirically that NPTS also does well non-asymptotically. Algorithm~\ref{alg:NPTS} gives the pseudo-code for NPTS. In what follows, we show that NPTS is a special case of DPPS. In NPTS, the arms are selected in each-round (see lines 9-10 in Algorithm~\ref{alg:NPTS}) based on the argmax of the weighted average of the observed rewards (weights drawn from a Dirichlet distribution). Interestingly, this is exactly the mean of a random distribution drawn from a Bayesian-Bootstrap (Eq.~\ref{def: BB-mean}), and Bayesian-Bootstrap is a special case of Dirichlet-processes (see Eq. \ref{def:Bayesian-Bootstrap}). Therefore, NPTS is a special case of DPPS, when the DP for each arm is taken to be the Bayesian-Bootstrap, and cannot account for prior knowledge (following our discussion in Section~\ref{sec:Background-DP} on Bayesian Bootstrap and DP posteriors).

\section{Numerical experiments} \label{sec:Numerics}

In this section, we exhibit empirical performance of DPPS on challenging Bernoulli bandit, Beta bandit, and a real-world agriculture dataset. In the experiments that follow, all regret plots exhibit average regret over 200 independent runs and $10\%-90\%$ quantile levels. For Bernoulli bandits we compare DPPS with Beta-Bernoulli Thompson sampling and UCB. Whereas for the other two environemnts we compare with UCB and a generalized version of Beta/Bernoulli~\citep{agrawal2013thompson} TS because it's difficult to implement usual parametric Thompson sampling in those settings (especially for the DSSAT bandit setting). Impressive performance of DPPS in a Gaussian bandit environment (with both mean and variance unknown to the algorithmic agent) is also shown in Sec.~\ref{sec:DPPS-Gaussian-bandit}. A discussion on the general choice of (hyper)parameters of DP priors ($\alpha$,$F_{0}$, and truncation level of DP prior) is given in Section~\ref{sec:DPPSHyperparameters}.  Corresponding code is provided in the supplementary material, and details on computational costs of DPPS are discussed in Section~\ref{sec:computational-complexity-DPPS-NPTS}. All simulations  were generated using the Python 3 language, and made use of the NumPy \citep{harris2020array}, SciPy \citep{virtanen2020scipy}, and Matplotlib \citep{hunter2007matplotlib} libraries.

\paragraph{Bernoulli and Beta bandits} Here we evaluate DPPS in a 6 arm Bernoulli bandit setting with means [0.3, 0.4, 0.45, 0.5, 0.52, 0.55].  Note that all means being close to 0.5 makes it a challenging setting. We compare performance
of DPPS with UCB and another algorithm which is tailor-made for Bernoulli bandit environment -- Beta/Bernoulli Thompson Sampling (TS). The prior for Beta/Bernoulli TS is set as
Beta(1,1) (uniform). The base measure of the DP prior is also set as Uniform distribution ($\texttt{Beta}(1,1)$) for all the arms. Fig.~\ref{fig:DPPS-UCB-TS-BernoulliBandit} shows the perfomance of all the algorithms. Clearly, DPPS does as well as Beta/Bernoulli TS. This is impressive because unlike Beta/Bernoulli TS , DPPS is unaware of the parametric class of the reward distribution (Bernoulli), and still performed as well as Beta/Bernoulli TS. With the same DP priors we also run DPPS in a Beta bandit environment (with same mean as the Bernoulli bandit setting and scale factor of 5). Fig.~\ref{fig:DPPS-UCB-TS-BernoulliBandit} (right) also shows performance of DPPS in this setting, and clearly  DPPS outperforms other baselines.

\begin{figure}[h]
\centering
\includegraphics[width=0.49\linewidth]{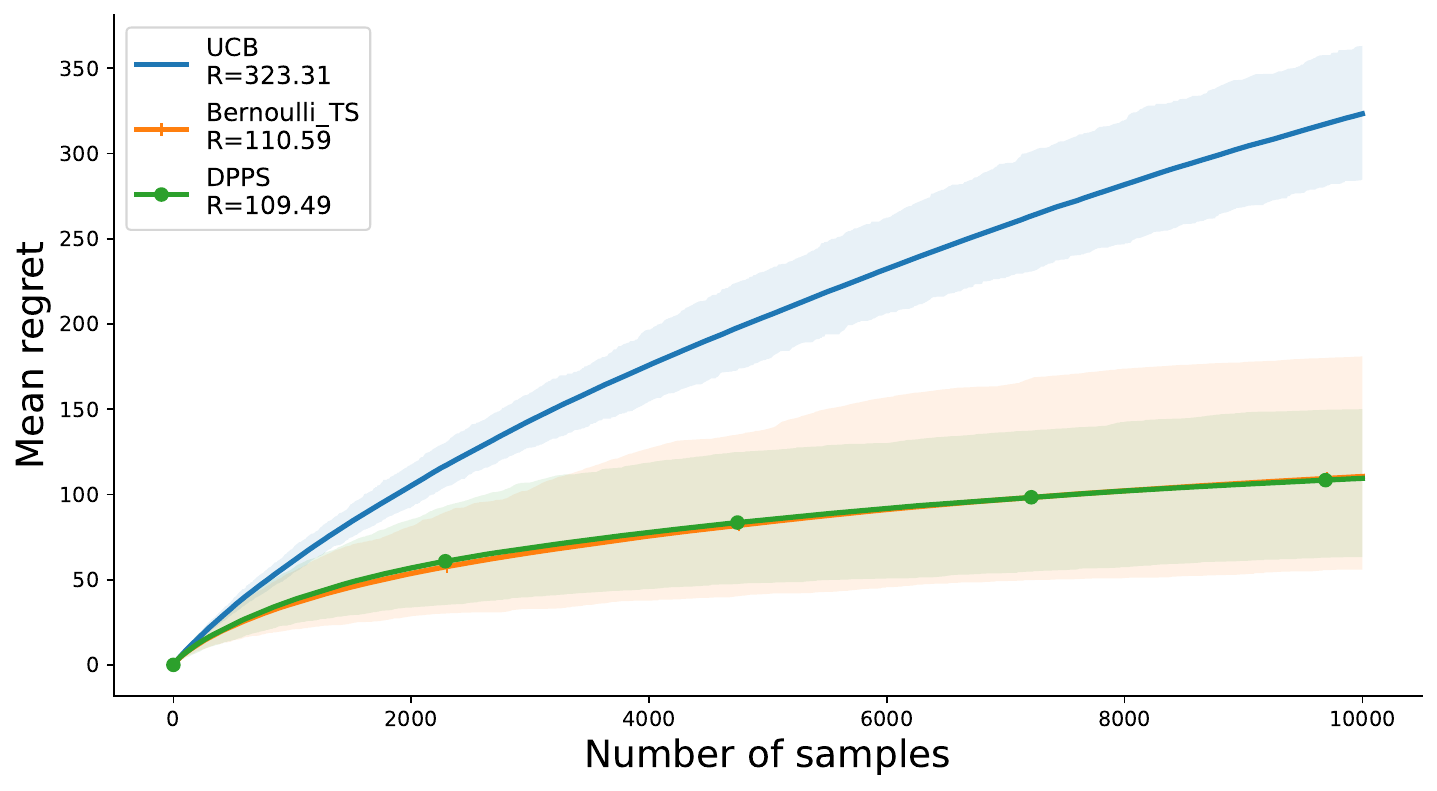}
\includegraphics[width=0.49\linewidth]{ 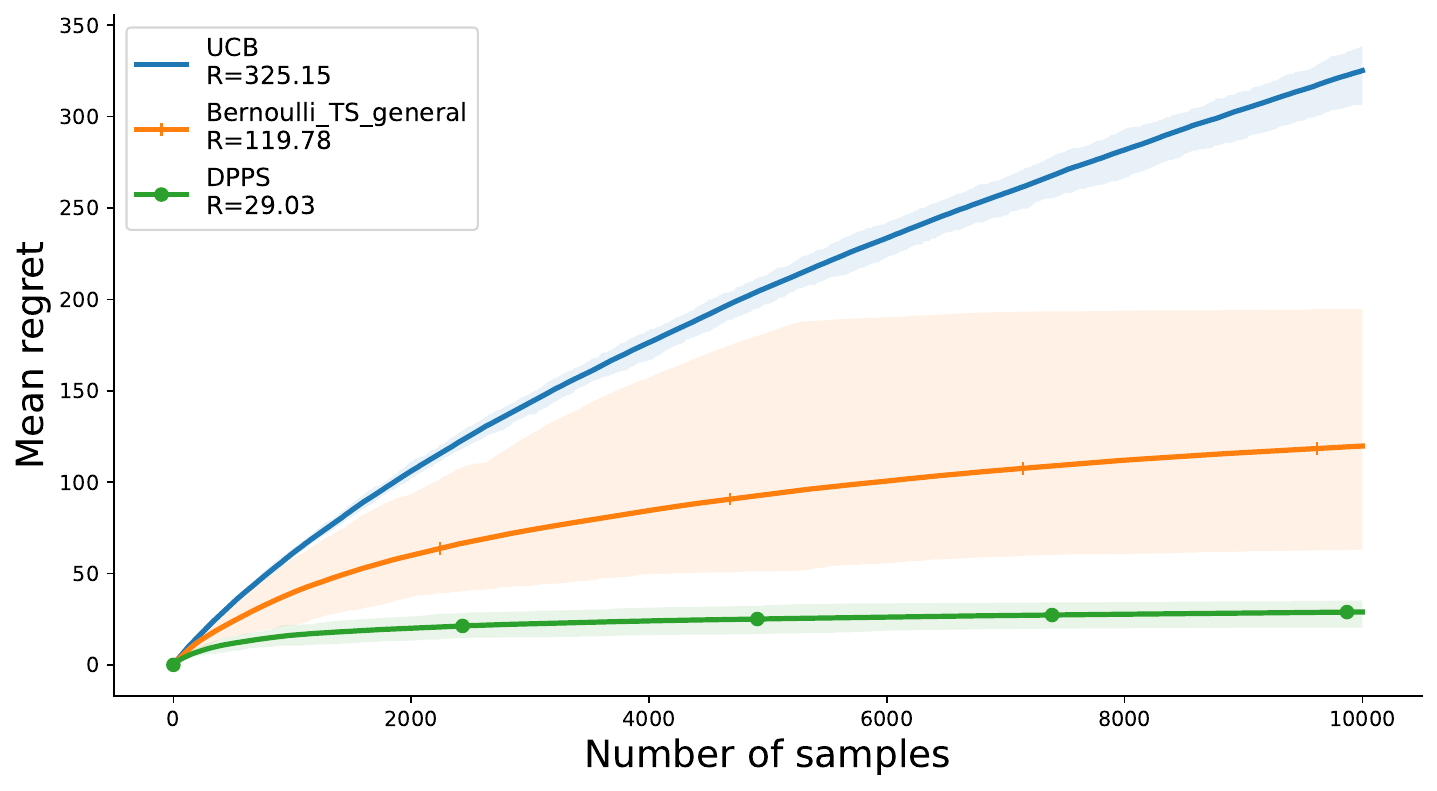}
\caption{Comparison of average regret in the Bernoulli bandit setting (left), and Beta Bandit setting (right) discussed in the text. }
\label{fig:DPPS-UCB-TS-BernoulliBandit}
\end{figure}

\paragraph{DSSAT bandits} Next, we illustrate the performance of DPPS on a challenging practical decision-making problem using the DSSAT-2 (Decision Support System for Agrotechnology Transfer) simulator ~\citep{hoogenboom2019dssat,gautron2022gym}. Harnessing more than 30 years of expert knowledge, this simulator is calibrated on historical field data (soil measurements, genetics, planting dates, etc) and generates realistic crop yields. Such simulations can be used to explore crop management policies in silico before implementing them in the real world, where their actual effect may take months or years to manifest themselves. More specifically, we model the problem of selecting a planting date for maize grains among 7 possible options, all other factors being equal, as a 7-armed bandit. The resulting distributions incorporate historical variability as well as exogenous randomness coming from a stochastic meteorologic model. In Figure~\ref{fig:generaalizedTS-DPPS-DSSAT}, we show distributions of crop yields generated from the DSSAT2 simulator. Note that these distributions are  right-skewed, multimodal and exhibit a peak at zero corresponding to years of poor harvest. Given this, they hardly fit to a convenient parametric model (e.g single-parameter-exponential-family, etc). Note that the distributions have bounded support and hence can be normalized to within $[0,1]$. Like for the Bernoulli bandit case, we use DP priors with uniform base measures ($\texttt{Beta}(1,1)$) for DPPS.

\begin{figure}[ht]
\centering
\includegraphics[width=0.49\linewidth]{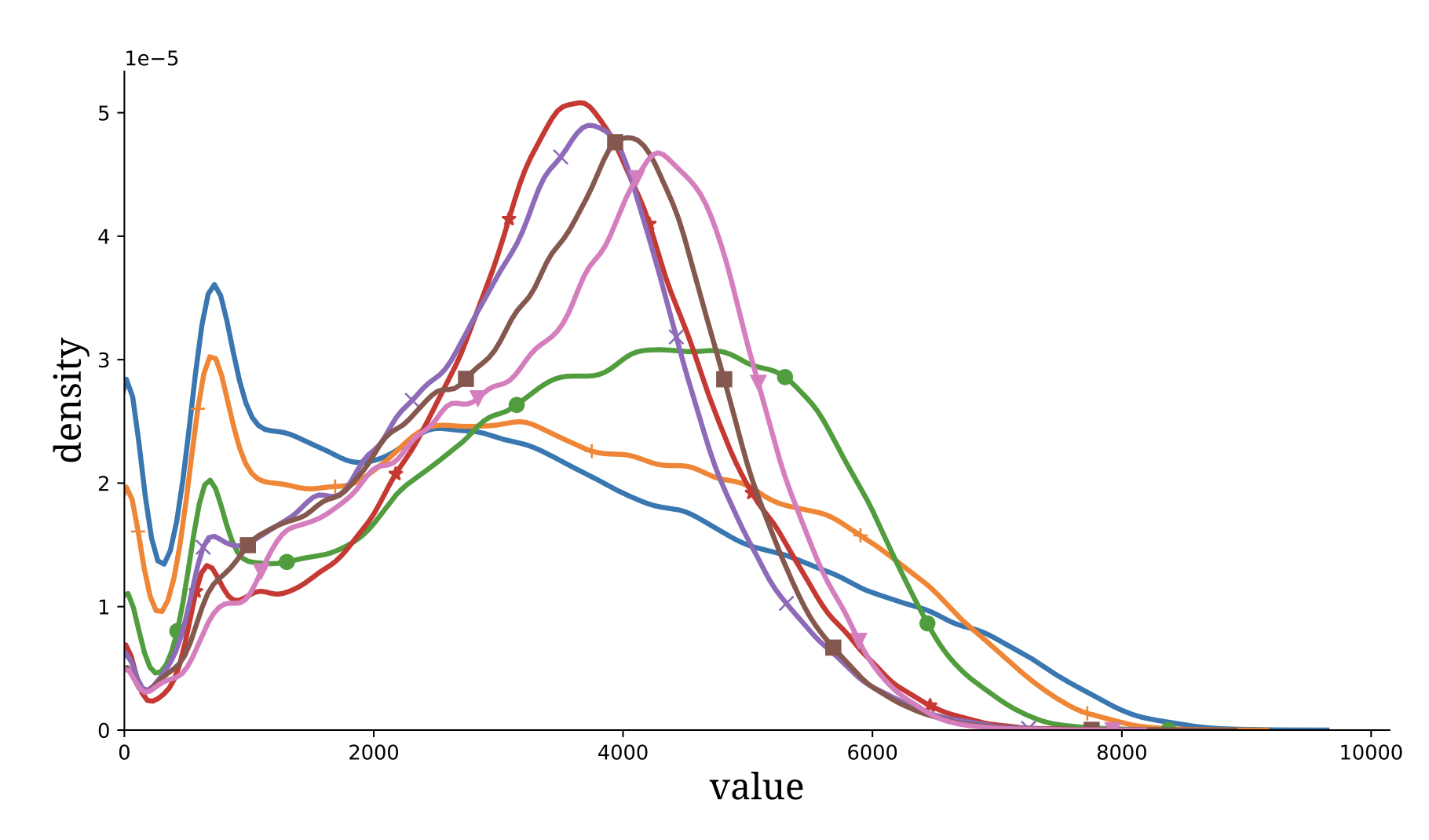}
\includegraphics[width=0.49\linewidth]{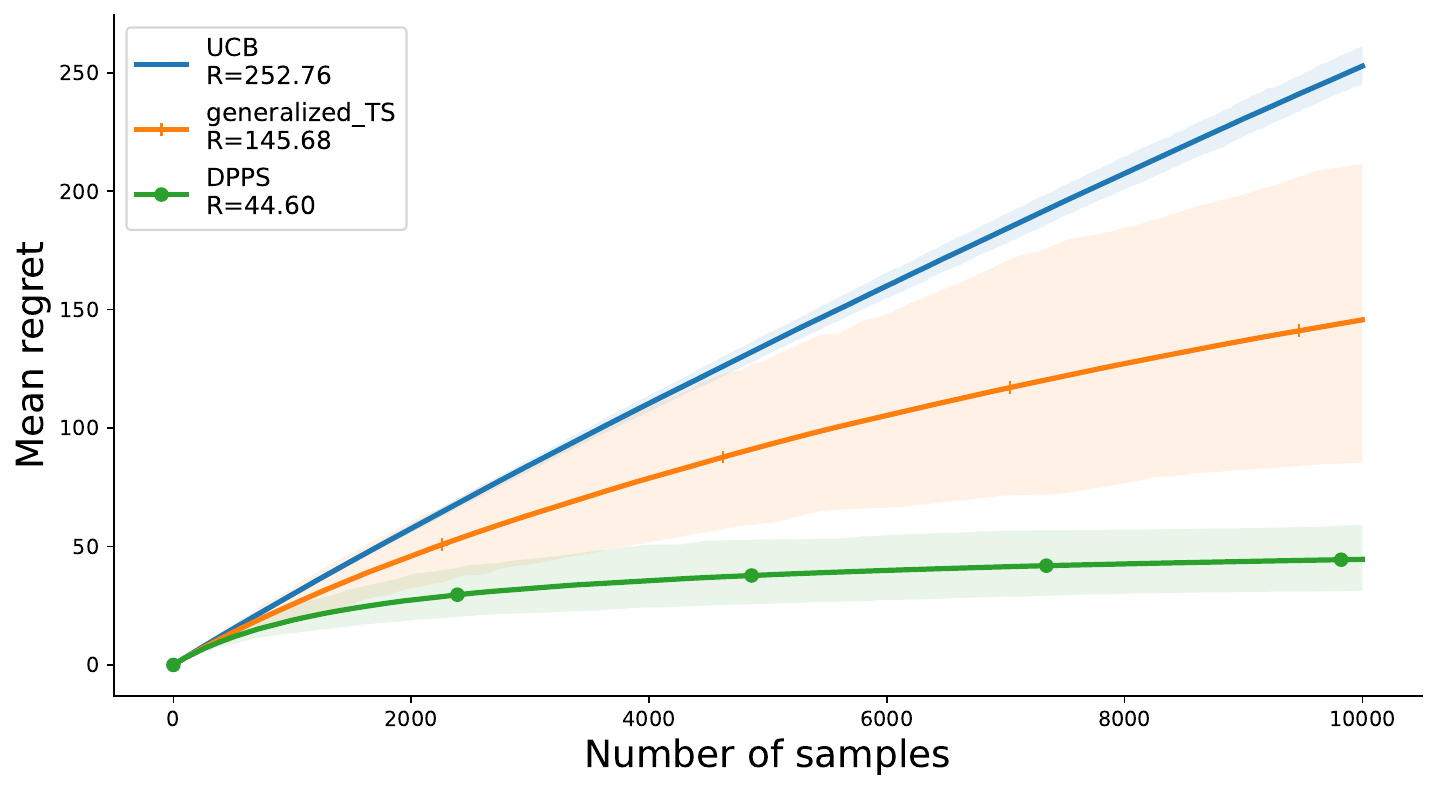}
\caption{Reward distributions from DSSAT simulator (left) and regret performances of bandit strategies (right) in the DSSAT environment.}
\label{fig:generaalizedTS-DPPS-DSSAT}
\end{figure}

Since a vanilla version of Thompson sampling is no longer feasible for DSSAT environment, we instead compare DPPS against a version of Beta/Bernoulli Thompson sampling, introduced in \cite{agrawal2013thompson}, that is adapted for general stochastic rewards based on a Bernoulli trial in each round with the obtained rewards as the mean parameter of the Bernoulli random variable. The same $\texttt{Beta}(1,1)$ prior is used for generalized TS as well. Fig.\ref{fig:generaalizedTS-DPPS-DSSAT} clearly shows DPPS outperforming generalized TS and UCB by a huge margin, and this example highlights the strength of DPPS as Bayesian nonparametric algorithm over it's closest parametric-counterpart of generalized TS. Note that so far we used agnostic base measures for the DP priors ($\texttt{Beta}(1,1)$), i.e. these base measures (and hence the corresponding DP priors) do not convey any special knowledge about the bandit environment. However, DPPS allows for encoding this prior knowledge about the bandit environment through base-measures of the DP priors, and we illustrate this next using a simple example.

\begin{figure}[h]
\centering
\includegraphics[width=0.55\linewidth]{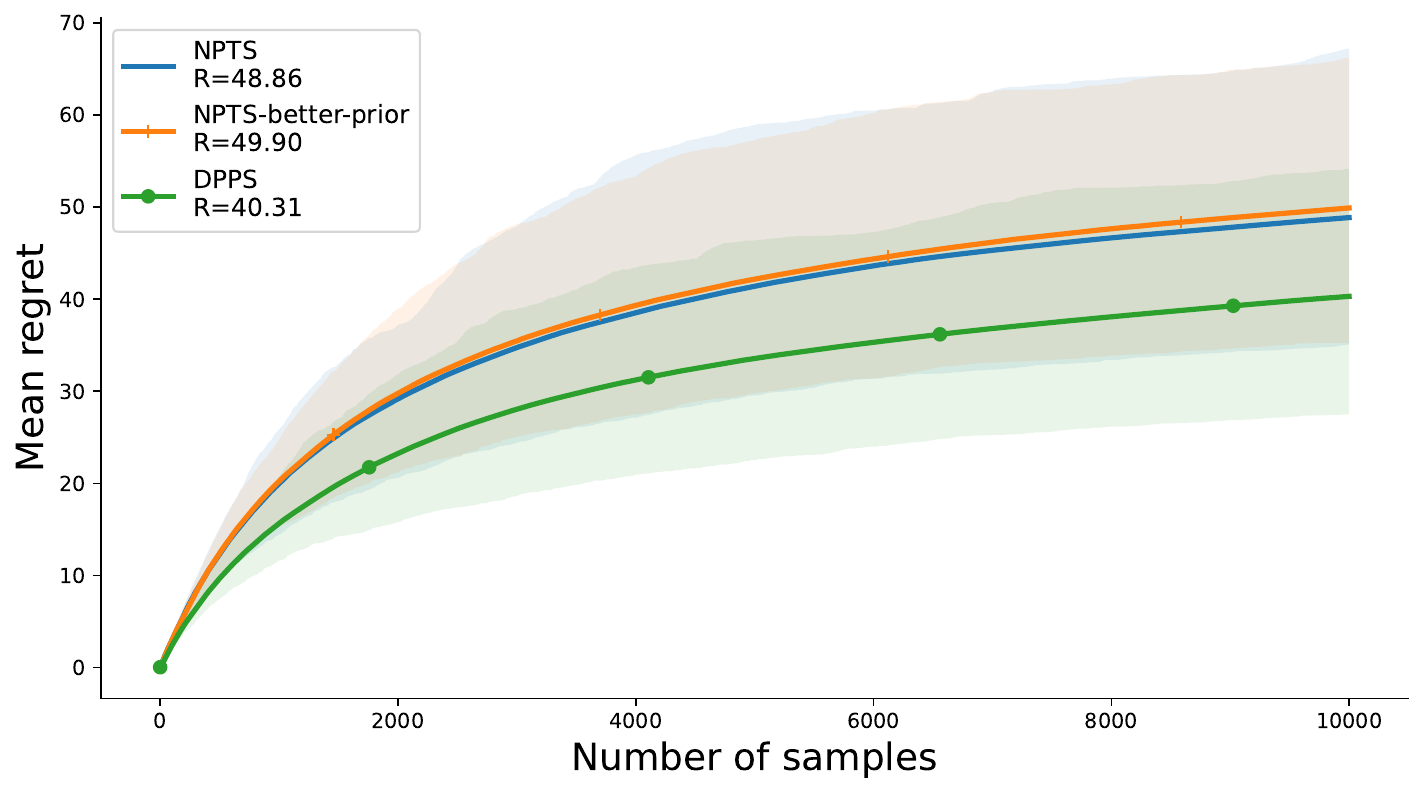}
\caption{Average regret in the DSSAT bandit environment with beneficial priors for both NPTS and DPPS.}
\label{fig:NPTS-DPPS-better-priors}
\end{figure}

\paragraph{Incorporating prior knowledge through DPPS} 
Recall from Sec.~\ref{subsec:NPTS-DPPS} that NPTS is a special case of DPPS in the Bayesian Bootstrap limit of the DP prior. Therefore, the base measure for NPTS for a particular arm is empirical CDF of the reward distributions based on current observations for that arm, beginning with some \textit{pseudo-rewards/artificial-history} for each of the k-arms. Given that the base measure is an empirical CDF, in NPTS, it's not possible to utilize even some first order prior information about the bandit environment that may be available. This is, however, possible in general cases of DPPS through the continuous base measures of DP priors. This can be clearly exhibited through a  simple example. We start DPPS with a more informed choice of priors, i.e. instead of  $\texttt{Beta}(1,1)$ base measure for the DP priors for all the arms, we express more confidence in the third (optimal) arm by using $\texttt{Beta}(1,0.1)$ as base measure for this arm. We compare this with a version of NPTS that starts with pseudo-rewards of $X_{k} = 0.01$ for all but the third arm (for which it uses a value of 1).  Fig.~\ref{fig:NPTS-DPPS-better-priors} confirms better performance of DPPS with this choice of DP priors, and no change in performance of NPTS even with initial condition that heavily favors the third arm.

\section{Information theoretic analysis of DPPS} \label{sec:Theoretcial-analysis}
In Section~\ref{sec:DPPS}, we saw that, like Thompson sampling, DPPS is a probability matching algorithm. In this section, we  utilize this property of DPPS to derive upper bounds on its Bayesian-regret. Particularly, we generalize the information theoretic analysis of Thompson sampling introduced in \cite{russo2016information} to a wider class of probability matching algorithms, and show that DPPS is encapsulated in that generalization.  We begin by summarizing the key-steps in the analysis of ~\cite{russo2016information} and also include complete proofs for the sake of completion in Sec.~\ref{sec:Proofs}. 

\subsection{Information theoretic analysis of Thompson sampling}

\cite{russo2016information} introduced an elegant framework for deriving upper bounds on Bayesian regret of Thompson sampling.  Firstly, the Bayesian regret is re-expressed in terms of the entropy of the posterior distribution of optimal action, and an upper bound on \textit{information ratio}, 

\begin{thm}\citep{russo2016information}\label{fact:Bayesian-regret-information-ratio}
For any $T \in N$, provided that $\Gamma_{t} \le \Gamma$ almost surely for each $t \in {1,..,T}$, $\E\left[\mathrm{Regret}(T,\pi^{TS})\right] \le \sqrt{\Gamma \mathbb{H}(\alpha_{1}) T}$.
\end{thm}

The information ratio, $\Gamma_{t} := \frac{(\E_t[R_{t,A_*} - R_{t,a}])^{2}}{I_{t}(A^\star;R_{t,a})}$, is defined as the  ratio of the square of the instantaneous expected regret by choosing action $a$ to the instantaneous \textit{information gain} about optimal action $A^\star$ if action $a$ is chosen. Clearly, bounding Bayesian regret of an algorithm then boils down to bounding the information-ratio of that algorithm. Particularly, for Thompson-sampling,  in $\sigma$-sub-Gaussian reward noise bandit setting, one obtains the following bound,

\begin{lem}\citep{russo2016information}\label{fact:Information-ratio-MAB-for-sub-gaussian-rewards}
\[
\Gamma_{t} \le 2|\mathcal{A}|\sigma^{2}.
\]
\end{lem}

This bound when combined with Theorem~\ref{fact:Bayesian-regret-information-ratio} and upper bound of $\log K$ for entropy of any posterior distribution of optimal action leads to the following bound on the Bayesian regret of Thompson sampling,

\begin{thm}\citep{russo2016information}
\label{thm:TS-Bayesian-regret}
\[\E\left[\mathrm{Regret}(T,\pi^{TS})\right] \le \sigma\sqrt{2{K}(\log{K})T}\,,\]
\end{thm}

The proof of Lemma~\ref{fact:Information-ratio-MAB-for-sub-gaussian-rewards} hinges on two crucial steps, and we highlight those, referring the reader to Sec.~\ref{sec:Proofs} for other details. First, re-writing of the instantaneous per-step Bayesian regret by utilizing the probability matching property of Thompson sampling, $\mathbb{P}_{t}(A^\star=a)= \mathbb{P}_{t}(A_{t}=a)$, as follows,
\begin{align}\label{Eq:Regret-decomposition}
\E_{t} \left[ R_{t,A^\star} - R_{t,A_{t}}\right] &= \sum_{a\in \A} \Prob_{t}(A^\star=a) \E_{t}\left[R_{t,a} | A^\star=a \right]- \sum_{a\in \A}\Prob_{t}(A_{t}=a)\E_{t}[ R_{t,a} | A_t=a] \ \\
&= \sum_{a\in \A} \Prob_{t}(A^\star=a)\left( \E_{t}\left[R_{t,a} | A^\star=a \right] - \E_{t}[ R_{t,a}]\right). \nonumber
\end{align} 

Second, bounding this instantaneous per-step regret by bounding $\left( \E_{t}\left[R_{t,a} | A^\star=a \right] - \E_{t}[ R_{t,a}]\right)$, This is done by an application of the variational formula~\citep{cover1999elements} for the KL divergence, ${\KL}(P||Q)$, between two absolutely continuous measures, $P$ and $Q$,
\begin{fact}\citep{cover1999elements}
\label{fact:KL-var}
\[ \KL(P||Q)=\mathrm{sup}_{X}\{\E_{P}[X] - \log\E_{Q}[\mathrm{exp}\{X\}]\}. \]
\end{fact}

If we substitute, the random variable, $X \equiv X(t) = R_{t,a} -\E_{t}[R_{t,a}]$, i.e. the instantaneous reward noise, with $P = \Prob_{t}(R_{t,a} | A^\star=a)$ and $Q = \Prob_{t}(R_{t,a})$ in the above variational formula, and when $X(t)$ is $\sigma$-sub-Gaussian, it's easy to obtain the following bound,

\begin{lem}\citep{russo2016information}\label{fact:KL-var-X}
\[
\E_{t}\left[R_{t,a} | A^\star=a \right] - \E[R_{t,a}] \le \sigma \sqrt{2\KL(\Prob_{t}(R_{t,a} | A^\star=a)||\Prob_t(R_{t,a}))}. \]
\end{lem}

Substituting the result of Lemma~\ref{fact:KL-var-X} in Eq.~\ref{Eq:Regret-decomposition}, and utilizing the definition of information gain, $I_{t}(A^\star;R_{t,a})$, and information ratio, $\Gamma_{t}$,  yields the bound in Lemma~\ref{fact:Information-ratio-MAB-for-sub-gaussian-rewards}.

\subsection{A generalization of the information theoretic analysis in \cite{russo2016information} } \label{subsec:Admissible-ProbMatch}

\paragraph{Choice of $\Prob_{t}(A^\star=a)$} It's easy to notice in the preceding analysis (specifically Eq.~\ref{Eq:Regret-decomposition}) that in the analysis of \cite{russo2016information} there's explicitly no restriction on  $\Prob_{t}(A^\star=a)$ for it to be derived using a Bayes-rule based posterior-distributions  of arm-rewards, $\Prob_{t}(R_{t,a})$, as is done in parametric Thompson sampling.  However, this choice is rather implicit, given the decision theoretic and information theoretic \textit{coherency} of {Bayesian framework}~\citep{wald1961statistical,zellner1988optimal}. 
Moreover, Bayesian-framework is not limited to Bayes-rule based derivation of posterior distributions. Another \textit{valid} Bayesian approach~\citep{orbanz2009construction,ghosal2017fundamentals} for obtaining posteriors is leveraging the property of \emph{conjugacy} as discussed in Sec~\ref{sec:Background-DP}.  In particular, most {nonparametric} priors do not satisfy the necessary conditions for Bayes rule (See section~\ref{subsec:Bayes_rule}), and one relies on their conjugacy property to derive the corresponding posteriors. An elegant way to characterize the class of {valid} Bayesian approaches is based on martingales of probability measures, and we refer the interested reader to  papers concerning \textit{predictive-Bayes} ~\citep{holmes2023statistical, fong2023martingale, fortini2025exchangeability} for details.

\paragraph{Admissible probability matching algorithms} From the preceding discussion, we can conclude that all probability matching algorithms which derive  $\Prob_{t}(R_{t,a})$ (and hence $\Prob_{t}(A^{*}=a)$) using a valid Bayesian approach are \textit{admissible} in the information theoretic analysis of \cite{russo2016information}. Additionally, these admissible algorithms would enjoy same bounds  on their information-ratio and (consequently) Bayesian regret as those for parametric Thompson sampling (i.e. Lemma~\ref{fact:Information-ratio-MAB-for-sub-gaussian-rewards} and Theorem~\ref{thm:TS-Bayesian-regret}), if they satisfy certain  \textit{auxiliary conditions} required from the analysis. For the case of $\sigma$-sub Gaussian reward noise discussed before, we list these next.

\paragraph{Auxiliary conditions} It is easy to see that we require the following auxiliary conditions for an admissible probability matching algorithm to enjoy same bounds on Bayesian regret as those for parametric Thompson sampling in \citep{russo2016information}: In each round $t$, 
\begin{enumerate}
\item The instantaneous reward noise, $X(t)$, in Lemma~\ref{fact:KL-var-X}, is a $\sigma$-sub-Gaussian random variable,   
\item $\KL(\Prob_{t}(R_{t,a} | A^\star=a)||\Prob_t(R_{t,a}))$ in Lemma~\ref{fact:KL-var-X} is well defined.

\end{enumerate}

The second condition holds if $\Prob_{t}(A^\star=a)>0$ owing to a classical fact in conditional probability, 

\begin{fact}~\citep{williams1991probability}
\label{fact:abs-continuity}
For any random variable $Z$ and event $E \subset \Omega$, where $\Omega$ is the probability space, if $\Prob_{t}(E) = 0$, then $\Prob_{t}(E|Z) = 0$ almost surely. Conversely, for any $x \in \mathcal{X}$ with $\Prob_{t}(X = x) > 0$, $\Prob_{t}(Y|X = x)$ is absolutely continuous with respect to $\Prob_{t}(Y)$.
\end{fact}

Next, we show that DPPS is an admissible probability matching algorithm that satisfies the auxiliary conditions listed above.

\subsection{Bayesian regret of DPPS}
That DPPS is an {admissible} probability matching algorithm is easy to see: DPPS utilizes a valid Bayesian approach, i.e. conjugacy of DP priors/posteriors, to derive $\Prob_{t}(A^\star=a)$; Also DPPS satisfies both the auxiliary conditions discussed in previous sub-section;  For condition (2), clearly,  $\Prob_{t}(A^\star=a) > 0$ for DPPS,  whenever the base measure, $F_{0}$, of the DP prior (and hence of the corresponding DP posterior), $\DP(\alpha,F_{0})$, is non-null. For condition (1), the following property of the tail of random measures sampled from DP priors/posteriors ensures $\sigma$-sub-Gaussian nature of the instantaneous reward noise, $X(t)$, whenever the base measure, $F_{0}$, of the DP prior, $\DP(\alpha,F_{0})$, is $\sigma$-sub-Gaussian,

\begin{fact}[\cite{doss1982tails}] Let $F \sim \DP(\alpha, F_{0})$, then almost surely the tails of $F$ and distributions sampled from the DP posterior of $F$, $\DP(\alpha + n, \overline{F_{n}})$, given samples $X_{1},. . ., X_{n}$, are dominated by (and are much smaller than) the tails of $F_{0}$.
\end{fact}

Therefore, DPPS is an admissible probability matching algorithm that satisfies the necessary auxiliary conditions under specific constraints on the DP priors, and hence whenever those constraints on the DP priors are satisfied, DPPS  enjoys the same upper-bounds on Bayesian regret as those for parametric Thompson sampling. More concretely,

\begin{thm}\label{thm:Bayesian-regret-DPPS}
For the setting of $\sigma$-sub-Gaussian arm reward distributions, starting with a DP-prior, $\DP(\alpha, F_{0})$, with $F_{0}$ as a $\sigma$ sub-Gaussian distribution, the Bayesian regret of DPPS satisfies,
\[\E\left[\mathrm{Regret}(T,\pi^{DPPS})\right] \le \sigma\sqrt{2{K}(\log{K})T}\,,\]
\end{thm}
where the expectation is taken over the randomness in the policy and the prior of the environment.  
\section{Conclusions and Perspectives}\label{sec:Discussions}

In this paper, we introduced a Bayesian non parametric algorithm based on Dirichlet processes, DPPS, for multi-arm bandits that combines the strength of (Bayesian) Bootstrap with a principled mechanism of incorporating and exploiting prior information about the bandit environment. DPPS enjoys similar optimality guarantees on Bayesian regret as parametric Thompson sampling, and among other advantages of DPPS over its parametric counterpart is its \textit{flexibility}. This is because the stick-breaking implementation of DPPS introduced in this paper can be used for different types of bandit environments, contrary to parametric Thompson sampling  whose implementations differ according to the bandit environment, and can easily lead to intractable posteriors (except for a few special cases) which need to be approximated using approximate inference schemes such as MCMC, variational inference, etc, and, if not done carefully, such approximate-inference based Thompson sampling has been shown to incur sub-optimal performance, even in simple settings~\citep{phan2019thompson}. Next, we discuss a few research directions. 

Firstly, we point that DPs are not the only Bayesian nonparametric priors on the space of distribution functions,   and further generalization of DPPS is possible.  For example, other probability matching algorithms using Pitman-Yor~\citep{pitman2006bessel} processes and P{\'o}lya-Tree priors ~\citep{castillo2017polya,castillo2024bayesian} can be useful generalizations of DPPS. Next, we consider DPPS as a generic \textit{design principle}, based on Bayesian non-parametric statistics, that can be extended to the setting of Markov Decision Processes as well, especially in the model-free scenario. Particularly, it seems that Bootstraped-DQN~\citep{osband2016deep} can be interpreted as a DP based algorithm in the noninformative limit of the DP posteriors (i.e. Bayesian Bootstrap), and it would be interesting to extend it  to a full-fledged DP implementation to account for any prior information~\citep{osband2018randomized} in a principled manner similar to that shown in this paper. Dirichlet Processes may also prove to be useful in efficiently representing value distributions in distributional reinforcement learning~\citep{bellemare2023distributional}. 
We leave these intriguing research questions and extensions for future work.




\section*{Acknowledgments}

This work has been supported by the French Ministry of Higher Education and Research, the Hauts-de-France region, Inria, the MEL, the I-Site ULNE regarding project RPILOTE-19-004-APPRENF, the French National Research Agency under PEPR IA FOUNDRY project (ANR-23-PEIA-0003). S.Vashishtha is partially supported by the ANR program “AI PhD@Lille”.  We thank Emilie Kaufmann and other members of our group (SCOOL) for encouraging discussions. S. Vashishtha thanks Isma{\"e}l Castillo for their lectures on Bayesian nonparametric statistics at the  51st Saint Flour Probability summer school, July 2023.  


\beginSupplementaryMaterials

\section{General Bayesian framework}\label{sec:Bayesian-framework}

In this section, we highlight a generalized Bayesian framework, and the conditions for existence of posteriors and, when they exist, methods of deriving posteriors from priors. Most of these results are standard in Bayesian-non-parametric statistics, and we refer the reader to \cite{ghosal2017fundamentals, orbanz2009construction} for details.

A general Bayesian modeling problem can be formulated as follows. We choose prior $Q$ on parameter $\Theta \in \mathbf{T}$ and the observation model $M$ as $P_{\Theta}$, observation space as $\mathbf{X}$. To summarize, both Bayesian and non-parametric Bayesian models can be written as follows, 
\begin{align}\label{eq:Bayesian-framework}
\Theta \sim Q,\\ 
X_{1},. . ., X_{n}|\Theta \sim P_{\Theta} \label{eq:Bayesian-framework2}
\end{align}

Whereas for Bayesian parametric models the parameter space $\mathbf{T}$ is finite-dimensional (e.g. $\mathbb{R}^{d}$), it's infinite for Bayesian non-parametric models. Thus in order to define a non-parametric Bayesian model, we have to define a probability distribution (the prior) on an infinite-dimensional space. A distribution on an infinite-dimensional space $\mathbf{T}$ is a stochastic process with paths in $\mathbf{T}$. 

For more clarity, the DP model can be  re-written in the framework of Eqs.~\ref{eq:Bayesian-framework}-\ref{eq:Bayesian-framework2} as follows, 
\begin{align}\label{eq:Bayesian-framew}
\Theta \sim DP(\alpha,G_{0}),\\ 
X_{1},. . ., X_{n}|\Theta \sim \Theta
\end{align}

The goal in Bayesian (both parametric and nonparmetric) inference is to figure out the posterior which is a probability kernel given as,
 \[q[\cdot, x] = \mathbb{P}(\Theta \in \cdot|X=x). \]

For existence of $q$ the following is required,

\begin{thm}
If $\mathbf{T}$ is a standard Borel space, $\mathbf{X}$ a measurable space, and a Bayesian model is specified as in Eqs.~\ref{eq:Bayesian-framework}, the posterior $q$ exists
\end{thm}

Having established the existence properties, let's discuss different ways of obtaining posteriors, given observations. In Bayesian framework, there are two ways, Bayes rule and Conjugacy, and we give existence results for each of these,

\subsection{Bayes-rule}
\label{subsec:Bayes_rule}
It's a popular update rule, however it's not always possible to use Bayes-rule for obtaining posteriors. The following theorem makes it concrete,

\begin{thm}\label{thm:Bayes-rule}
(Bayes’ Theorem). Let $\mathbf{M} = P( \cdot , \mathbf{T})$ be an observation model and $Q \in PM(T) $) a prior (PM denotes space of probability measures on $\mathbf{T}$). Require that there is a $\sigma$-finite measure $\mu$ on $\mathbf{X}$ such that $P( \cdot , \Theta) \ll \mu$ for every $\Theta \in \mathbf{T}$. Then the posterior under conditionally i.i.d. observations $X_{1}, ..., X_{n}$ is given as below, and $\Prob\{{P(X_{1},...,X_{n}) \in {0,\infty}}\} = 0$
\[ Q(d\Theta|X_{1}=x_{1},. . ., X_{n}=x_{n}) = \frac{\prod_{i=1}^{n} P(x_{i}|\Theta)}{P(X_{1},. . .,X_{n})} Q(d\Theta)  \]
\end{thm}
\subsection{Conjugacy}
\label{subsec:conjugacy}

For most non-parametric priors, the important absolute continuity  condition in Theorem~\ref{thm:Bayes-rule} doesn't hold, and hence Bayes' rule is not applicable. For example, If $\mathbb{P}[d\Theta|X_{1:n}]$ is the posterior of a Dirichlet process, then there is no $\sigma$-finite measure $\nu$ which satisfies $\mathbb{P}[d\Theta|X_{1:n} = x_{1:n}] \ll \nu$ for all $x_{1:n}$. In particular, the prior does not, and so there is no density $P(\Theta|x_{1:n})$~\citep{ghosal2017fundamentals}. In order to remedy this curse on non-parametric priors, the most important alternative to Bayes theorem for computing posterior distributions is conjugacy. Suppose $\mathbf{M}$ is an observation model, and consider now a family $\mathcal{Q} \subset PM(\mathbb T)$ of prior distributions, rather than an individual prior. We assume that the family $\mathcal{Q}$ is indexed by a parameter space $\mathbf{Y}$, that is, $\mathbf{M} = \{Q_y|y \in \mathbf{Y}\}$. Many important Bayesian models have the following two properties:

\begin{itemize}
    \item The posterior under any prior in $\mathcal{Q}$ is again an element of $\mathcal{Q}$; hence, for any specific set of observations, there is an $y' \in \mathbf{Y}$ such that the posterior is $Q_{y'}$
    \item The posterior parameter $y'$ can be computed from the data by a simple, tractable formula.
\end{itemize}

The above two points define the property of conjugacy. We saw in the main paper that DP priors enjoy conjugacy, and saw the simple update formula for the posterior, that resulted thanks to this property of conjugacy. For more details, we refer the reader to~\citep{orbanz2009construction}.

\section{Finite Stick breaking representation of Dirichlet Process priors}
\label{sec:stick-breaking-finite}

The finite stick-breaking representation of DP priors discussed in the main paper (Eqs.\ref{def:SB-DP}-\ref{def:SB-DP-3}) has been pivotal in the success of DP based Bayesian-nonparametric models. A major reason for this success is that such truncated representation is provably efficient~\citep{ishwaran2001gibbs}. Particularly, to quantify the accuracy loss owing to truncation consider the quantities, $T_{N} = (\sum_{N}^{\infty} q_{i})^{r}$ and $U_{N} = \sum_{N}^{\infty} q_{i}^{r}$, where $N$ is the level at which the representation is truncated,
\begin{align}
\E(T_{N}(r,a,b)) &= \left(\frac{\alpha}{\alpha+r}\right)^{N-1} ,\\
\E(U_{N}(r,a,b)) &= \left(\frac{\alpha}{\alpha+r}\right)^{N-1} \frac{\Gamma(r)\Gamma(\alpha + 1)}{\Gamma (\alpha + r)} 
\end{align}

Notice that both expressions decay exponentially fast in $K$, and hence good accuracy is achieved for moderate $K$. Fig.~\ref{fig:DP-alphalow-high} shows an application of this scheme to sample random measures from a DP prior, $\mathrm{DP}(\alpha, F_{0})$  for two different values of concentration parameter, $\alpha$. In order to give more intuition to appreciate the utility of DPs for nonparametric inference,  We given an example on inference on a galaxy-dataset. We also used this (and some other) benchmarks to validate the performance of our StickBreaking module for DPPS. 

\begin{figure}[H]
\centering
\includegraphics[width=0.48\linewidth]{ 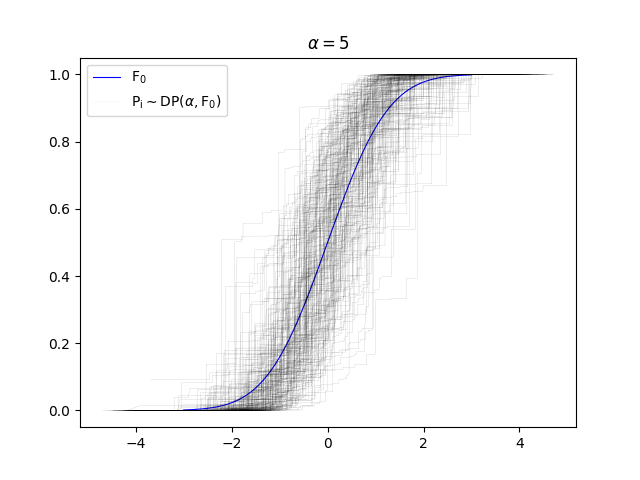}
\includegraphics[width=0.48\linewidth]{ 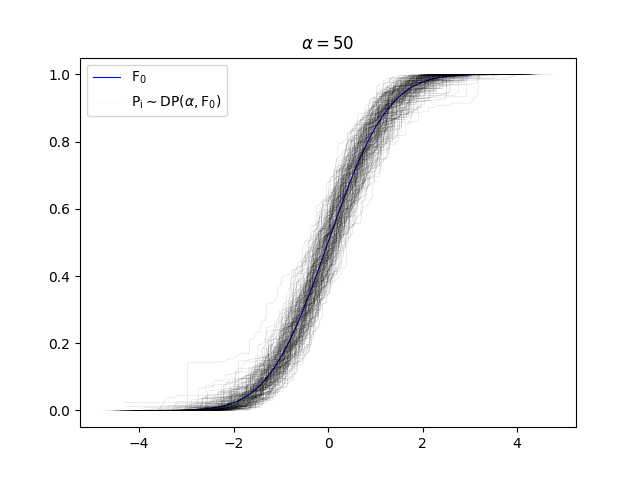}
\caption{200 random measures sampled from $\mathrm{DP}(\alpha,\mathrm{F_{0}})$ where $\alpha = 5$ (left) and $50$ (right), $\mathrm{F_{0}} = N(0,1)$}
\label{fig:DP-alphalow-high}
\end{figure}

\paragraph{DPs for galaxy data-set}
We illustrate the application of Dirichlet processes for density estimation on a data set from the astronomy literature~\citep{roeder1990density}. The measurements are velocities at which galaxies in the Corona-Borealis region are moving away from our galaxy. If the galaxies are clustered, the velocity density will be multimodal, with clusters corresponding to modes. This happens to be the case, and the multi-modal nature is evident in the CDF of the data in Figure~\ref{fig:galaxy-dataset-DPprior/posterior} where the left and right regions of the CDF are almost flat, and most mass resides in the center. We estimate the CDF of this density using DP priors. Starting with a $\mathrm{DP}(\alpha=2,\mathrm{F_{0}}=\mathrm{N}(10,1))$ DP prior, we obtain a DP posterior, and the spread of distributions sampled from the DP posterior can be seen as confidence-set of the CDF of density estimated through Dirichlet process.

\begin{figure}[h]
\centering
\includegraphics[width=0.6\linewidth]{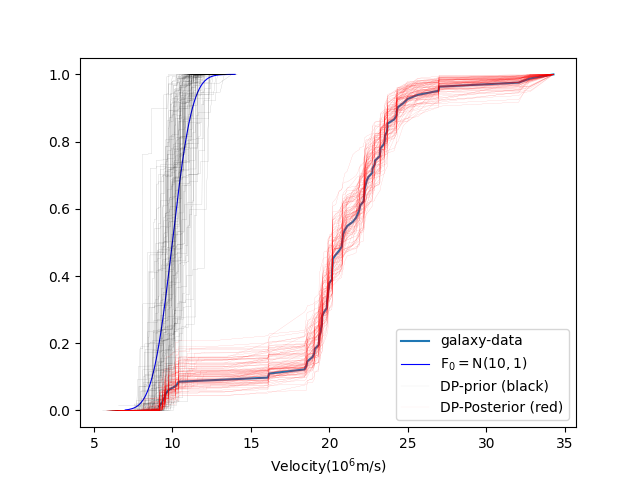}
\caption{DP prior and DP posterior compared against original galaxy dataset distribution.}
\label{fig:galaxy-dataset-DPprior/posterior}
\end{figure}

\section{DPPS for a Gaussian bandit with unknown means and variances}
\label{sec:DPPS-Gaussian-bandit}

\begin{figure}[H]
\centering
\includegraphics[width=0.55\linewidth]{ 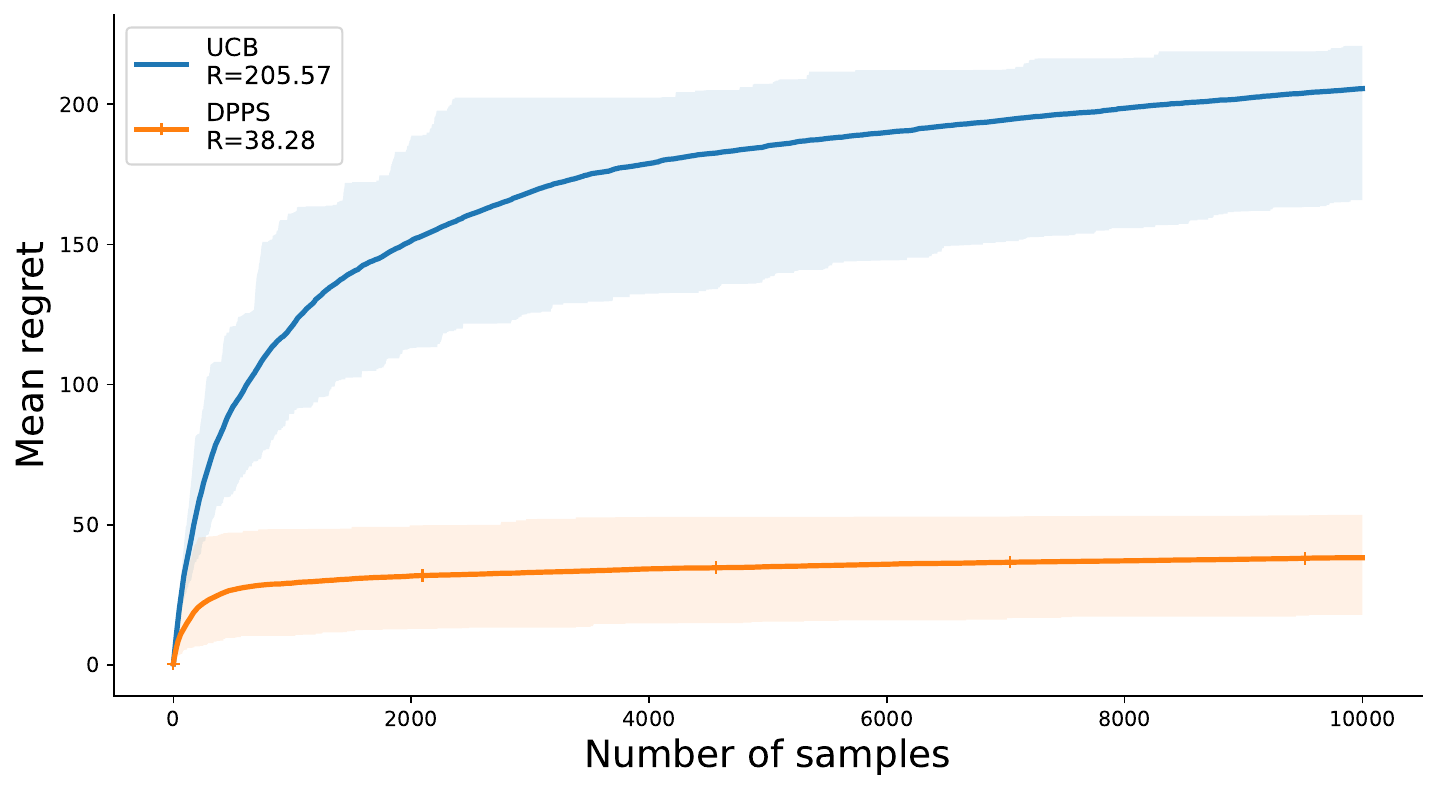}
\caption{DPPS for a challenging Gaussian bandit setting}
\label{fig:DPPS_Gaussian-bandit}
\end{figure}

A challenging  bandit setting is that of Gaussian bandit environment with both mean and variance of the underlying Gaussian distribution as unknown~\citep{cowan2018normal} to the bandit algorithm. Here we exhibit performance of DPPS in such a 7 arm Gaussian bandit environment $\{N(\mu_{k},\sigma_{k})\}_{k=1}^{K=7}$. The mean and variance of Gaussian bandit arms are sampled independently from a Gaussian such that $\mu_{k} \sim N(0,0.5)$  and $\sigma_{k} = |\psi_{k}|, \psi_{k} \sim N(0,0.5)$. Cumulative Regret averaged over 100 runs on one of the sampled instance of bandit environment is shown in Fig.~\ref{fig:DPPS_Gaussian-bandit}. Excellent performance of DPPS is evident. In this experiment, we chose $\alpha =2$, base measure of DP, $F_{0}$, as $N(0,0.5)$. 
 
\section{Choice of hyperparameters  in numerical experiments}
\label{sec:DPPSHyperparameters}

\begin{figure}[H]
\centering
\includegraphics[width=0.48\linewidth]{ 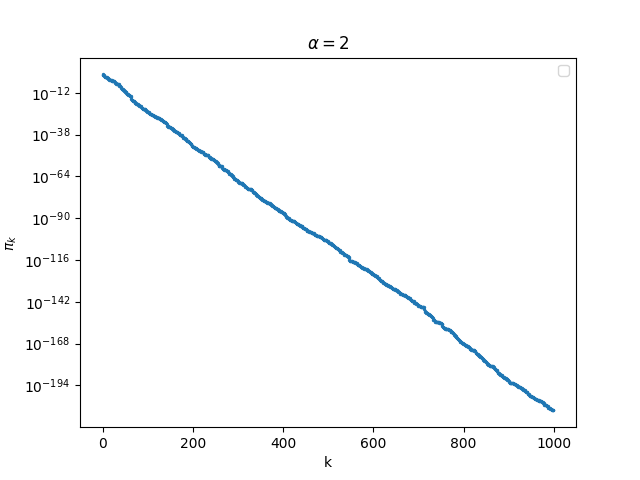}
\includegraphics[width=0.48\linewidth]{ 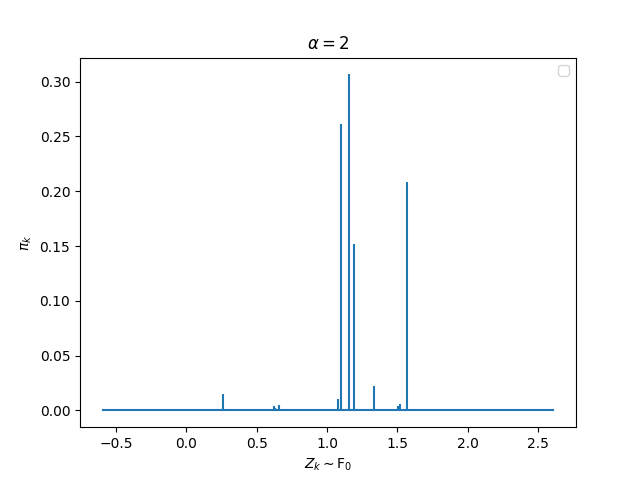}
\caption{Plot of first 1000 stick-breaking probability measure weights, $\pi_{k}$, for $\mathrm{DP}(\alpha=2,F_{0})$ with k (left) and with $Z_{k}\sim F_{0} (= N(0,1)$ (right)}
\label{fig:DP-alpha=2}
\end{figure}

\begin{figure}[H]
\centering
\includegraphics[width=0.48\linewidth]{ 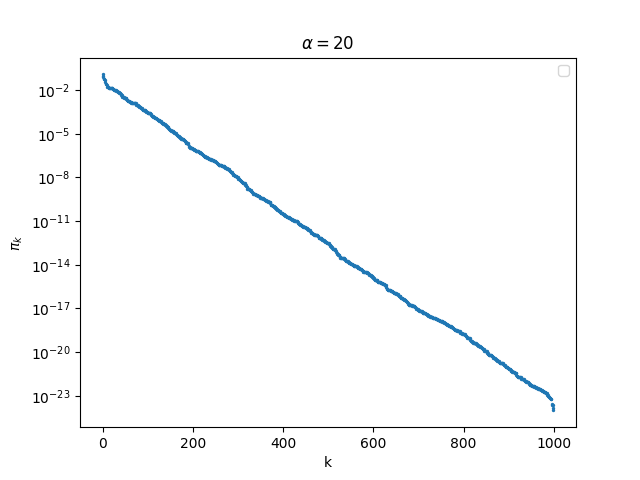}
\includegraphics[width=0.48\linewidth]{ 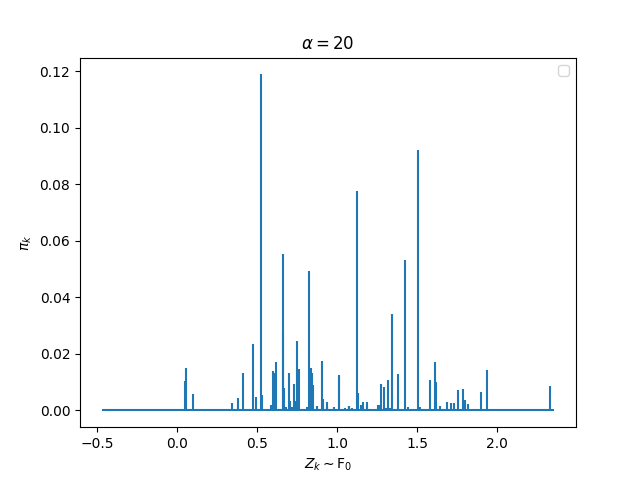}
\caption{Plot of first 1000 stick-breaking probability measure weights, $\pi_{k}$, for $\mathrm{DP}(\alpha=20,F_{0})$ with k (left) and with $Z_{k}\sim F_{0} (= N(0,1)$ (right)}
\label{fig:DP-alpha=20}
\end{figure}

Two hyperparameters in DPPS are $\alpha$ (concentration parameter) and $k_t$ (i.e. truncation level) in the stick breaking representation of DP prior (not the posterior), DP$(\alpha,F_{0})$.  We used $\alpha =2$ and $k_t = 100$ in all the experiments. Note that the choice of $\alpha$ directly influences the choice of $k_t$. This is because the number of weights $q_{i}$ in the stick breaking representation, $\sum q_{i} \delta_{x_{i}}$, carrying significant probability mass increase with increase in $\alpha$ ($V_{i}\sim$ Beta$(1,\alpha)$), and for higher $\alpha$ one needs to increase $k_t$.  For example, with $\alpha = 20$, we took $k_t =300$, and we got similar results, with a slight increase in computation cost though. An easy way to determine $k_t$ is to plot the stick breaking weights and remove stick breaking weights that are below a certain threshold (we chose $10^{-10}$ randomly). This relationship between $\alpha$ and stick breaking probability weights, $q_{i}$, can be seen in a simple example of $\mathrm{DP}(\alpha, F_{0})$ as shown in figs.~\ref{fig:DP-alpha=2}  and ~\ref{fig:DP-alpha=20}. Whereas for lower value of $\alpha$ only few weights have significant mass, for higher $\alpha$ the weights are more evenly spread compared to lower $\alpha$ case. 

\paragraph{Choice of base measure, $F_{0}$, of DP prior}

For choosing, $F_{0}$, the tail of the underlying reward distribution and a fact on the support of DPs is important. 
 
\begin{lem}[Support of DPs, see~\citep{ghosal2010dirichlet}] In the \textit{weak topology}, the support of DP$(\alpha,F_{0})$ is characterized as all probability measures $P^\star$ whose supports are contained in that of $F_{0}$
\end{lem}

Thus, choosing Beta(1,1) for a bandit problem with $\sigma = 10$, subGaussian noise is not a good idea. Similarly, theorem~\ref{thm:Bayesian-regret-DPPS} on Bayesian regret of DPPS, shows that choosing $F_{0}$ with $\sigma$-subGaussian tails corresponding to tails of the reward noise guarantees order optimal regret bounds.

\section{Computational costs of DPPS}
\label{sec:computational-complexity-DPPS-NPTS}
Improved performance and flexibility of DPPS (and other Bootstrap based algorithms such as NPTS) does come with higher computational cost. For example, in the 6-arm Bernoulli bandit environments of horizon $T =10000$, average run-time (over 200 independent runs) of DPPS was around 18 seconds, whereas that of parametric TS (conjugate prior/posterior) was 2-3 seconds. For the 7 arm DSSAT bandit problem, DPPS takes around 20 seconds, NPTS takes around 16 seconds.  Sec.~\ref{sec:computational-complexity-DPPS-NPTS} gives a detailed overview of the computational complexity of DPPS. All this said, this run time of DPPS can be significantly brought down by utilizing \textit{self-similarity}~\citep{ghosal2010dirichlet} of DP posteriors and parallel computation of DP posteriors that a construction exploiting this self-similarity would enjoy, which we plan to do in future.

Next, we detail the computational costs associated to a single-arm in each round. Let $n$ denote the number of observations for the arm. The important consideration in quantifying the running cost of DPPS is to scrutinize the posterior update step,
\beqa
\label{eq:multi-step-posterior-appendix}
Q_{n} &=& V_{n}\delta_{X_{n}} + \sum_{i=1}^{n-1}\left[V_{i}\prod_{j=i+1}^{n}(1-V_{j})\right]\delta_{X_{i}} + \left[\prod_{i=1}^{n}(1-V_{i})\right] Q_{0}
\eeqa

Here, one needs to sample $n$ beta random variables and have $\mathcal{O}(n)$ multiplications of these random variables, one for each of the past observations. Thus the running cost of DPPS is $\mathcal{O}(n)$ for each arm. DPPS also incurs a fixed memory and computational cost of $\mathcal{O}(K)$, sampling a DP prior, $Q_{0}$, where $K$ is the truncation level of the DP prior. Clearly, this additional but constant (in number of rounds and memory) cost is the difference between computational complexities of DPPS and NPTS (which needs similar $\mathcal{O}(n)$ multiplications between $\bf{X_{n}}$ and $\bf{W_{n}}\sim\mathrm{Dir}(n;1,. . .,1)$ random variables), and arises because of additional flexibility of DPPS in incorporating prior knowledge.

\section{Further related work}
\label{sec:Related works}

To the best of our knowledge, Dirichlet Processes in the context of bandits were first used in ~\citep{clayton1985bayesian} to study a version of the single-arm  Gittin's index problem, when the probability distribution of the arm is assumed to be DP distributed. Use of Bootstrapping for Thompson sampling seems to have appeared first in \cite{eckles2014thompson}, which was further improved and made more systematic in ~\citep{osband2015bootstrapped} where the authors also showed equivalence of Bootstrap-Thompson sampling (for Bernoulli-bandits) and Thompson sampling with Beta/Bernoulli priors in an exact sense, and speculated this equivalence for a wide class of bandit-environments if a proper mechanism for generating \textit{artifical history} (or prior information) could be identified. As shown in the current paper, DPPS provides a neat and principled mechanism for incorporating prior information (or generating artificial history), and generalizes this equivalence. Non-Parametric Thompson sampling (NPTS) and Multinomial Thompson Sampling (TS) were introduced in~\citep{riou2020bandit} without highlighting any concrete Bayesian connection of the former algorithm. NPTS was adapted for robustness in ~\citep{baudry2021optimality}. Some discussions concerning Bayesian interpretation of NPTS using DPs appeared in \cite{belomestny2023sharp} who provided a refined analysis of Multinomial TS. In many inference problems, mixture and diffusion priors can be useful to approximate the performance of Bayesian nonparametric priors~\cite{muller2013bayesian}, and a nice line of works performs Thompson sampling utilizing such (parametric) priors ~\citep{hong2022thompson,hong2020latent,kveton2024online}.  Aligning towards non-Bayesian side,  a sample mean based algorithm guaranteeing $O(\log N)$ instance-dependent regret appeared in \citep{agrawal1995sample}, a sub-sampling based algorithm was reported in \citep{baransi2014sub} and analyzed for a two-arm bandit setting; a nonparametric Bootstrap based algorithm  was reported in \citep{kveton2019garbage}, and regret bounds derived for a Bernoulli bandit environment. Finally, we would like to comment here that Information theoretic analysis presented in this paper can also be seen as a special case of well established PAC-Bayes approach, and we refer the reader to an excellent article (and references therein) for details ~\citep{alquier2024user}.

\section{Technical derivations}
\label{sec:Proofs}

This section gives proofs of lemmas in the main paper extracted here for completion from ~\citep{russo2016information}

\subsection{Proof of Lemma~\ref{fact:Bayesian-regret-information-ratio}}

\begin{lem}\label{prop: regret bound}
For any $T\in \mathbb{N}$, if $\Gamma_t \leq \overline{\Gamma}$ almost surely for each $t \in \{1,..,T\}$,  
$$\E \left[{\rm Regret}(T, \pi^{\rm TS})  \right] \leq  \sqrt{\overline{\Gamma} \H(\alpha_1) T}.$$
\end{lem}
\begin{proof} Recall that $\E_{t}[\cdot] = \E[\cdot | \hist]$ and we use $I_{t}$ to denote mutual information evaluated under the base measure $\Prob_{t}$. Then,
\begin{eqnarray*}
\E \left[{\rm Regret}(T, \pi^{\rm TS})  \right] \overset{(a)}{=} \mathbb{E} \sum_{t=1}^{T} \E_{t} \left[ R_{t, A^\star} -R_{t,A_t} \right]  &=& \mathbb{E}\sum_{t=1}^{T} \sqrt{\Gamma_{t} I_{t}\left( A^\star ; (A_t, R_{t, A_t})  \right)} \\ 
&\leq& \sqrt{\overline{\Gamma}}\left( \mathbb{E} \sum_{t=1}^{T} \sqrt{I_{t}\left( A^\star ; (A_t, R_{t, A_t})  \right)} \right) \\
&\overset{(b)}{\leq}& \sqrt{\overline{\Gamma}T \mathbb{E} \sum_{t=1}^{T} I_{t}\left( A^\star ; (A_t, R_{t, A_t})  \right)},
\end{eqnarray*}
where (a) follows from the tower property of conditional expectation, and (b) follows from the Cauchy-Schwartz inequality. We complete the proof by showing that expected information gain cannot exceed the entropy of the prior distribution. For the remainder of this proof, let $Z_t = (A_t, R_{t,A_t})$. Then, using tower rule of conditional expectations we have,
\[
 {\E_{t}}\left[ I_{t}\left( A^\star ;Z_t \right) \right] =  I\left( A^\star ; Z_t | Z_1,...,Z_{t-1}  \right),
\]
and therefore,
\begin{eqnarray*}
\mathbb{E}\sum_{t=1}^{T} I_{t}\left( A^\star ; Z_t  \right) =  \sum_{t=1}^{T} I\left( A^\star ; Z_t | Z_1,...,Z_{t-1}  \right)
&\overset{(c)}{=}&  I\left(A^\star \, ;\,  Z_1,...Z_{T}    \right) \\ 
&=& \H(A^\star) - \H(A^\star | Z_1,...Z_{T}) \\
&\overset{(d)}{\leq}&  \H(A^\star),
\end{eqnarray*}
where (c) follows from the chain rule for mutual information, and (d) follows from the non-negativity of entropy.
\end{proof}

\subsection{Proof of Lemma~\ref{fact:KL-var-X}}

\begin{proof}
Define the random variable $X(t)=R_{t,a}- \E_{t}\left[ R_{t,a}\right]$. Then, for arbitrary $\lambda \in \mathbb{R}$, applying Fact \ref{fact:KL-var} to $\lambda X$ yields
\begin{eqnarray*}
\KL \left(\Prob_{t}\left( R_{t,a}  | A^\star=a^\star \right) \, || \, \Prob_{t}( R_{t,a})\right) &\geq& \lambda \E_{t}\left[X | A^\star=a^\star \right] - \log \E_{t} \left[ \exp\{\lambda X \} \right]\\
&= & \lambda \left( \E_{t}[R_{t,a}| A^\star=a^\star]- \E_{t}\left[ R_{t,a}\right] \right)- \log \E_{t} \left[ \exp\{\lambda X \}\right] \\
&\geq&\lambda \left( \E_{t}[R_{t,a}| A^\star=a^\star]- \E_{t}\left[ R_{t,a}\right] \right) - (\lambda^2 \sigma^2 /2).
\end{eqnarray*}
Maximizing over $\lambda$ yields the result. 
\end{proof}

\subsection{Proof of Lemma~\ref{fact:Information-ratio-MAB-for-sub-gaussian-rewards}}

\begin{proof}   
\begin{eqnarray*}
\E_{t} \left[ R_{t,A^\star} - R_{t,A_{t}}\right]^2 &\overset{(a)}{=}&  \left(\sum_{a\in \A} \Prob_{t}(A^\star=a) \left( \E_{t}\left[R_{t,a} | A^\star=a \right] - \E_{t}[R_{t,a}]\right) \right)^2 \\
&\overset{(b)}{\leq}& |\A|  \sum_{a\in \A} \Prob_{t}(A^\star=a)^2 \left( \E_{t}\left[R_{t,a} | A^\star=a \right] - \E_{t}[ R_{t,a}]\right)^2    \\
&\leq& |\A|  \sum_{a, a^\star \in \A} \Prob_{t}(A^\star=a)\Prob_{t}(A^\star=a^\star) \left( \E_{t}\left[R_{t,a} | A^\star=a^\star \right] - \E_{t}[ R_{t,a}]\right)^2   \\
&\overset{(c)}{\leq}& \frac{|\A|}{2} \sum_{a, a^\star \in \A} \Prob_{t}(A^\star=a)\Prob_{t}(A^\star=a^\star){\KL}\left(\Prob_{t}(R_{t,a} | A^\star = a^\star)   \,\, || \,\, \Prob_{t}(R_{t,a})  \right)  \\
&\overset{(d)}{=}& \frac{|\A| I( A^\star ; R_{t,A_{t}})}{2}
\end{eqnarray*}
where (b) follows from the Cauchy--Schwarz inequality, (c) follows from Fact \ref{fact:KL-var-X}, and (a) follows from Eq.\ref{Eq:Regret-decomposition}and (d) from the standard definition of mutual-information. 
\end{proof}

\section{Bandit Algorithm in \cite{riou2020bandit} }

\begin{algorithm}[h]
\caption{Algorithm in~\cite{riou2020bandit}}\label{alg:NPTS}
\begin{algorithmic}[1]
\Require Horizon $T \ge 1$, number of arms $K \ge 1$
\For{$k=1. . .K$,}
\State Set $R_{k}:=[1]$, and $N_{k}:=1$
\EndFor 
\For{$t = 1...T$,}
\For{$k=1. . .K$,}
\State Sample ${\bf W}_k \sim$ Dir$(1_{N_{k}})$ where $1_{N_k} = \underbrace{(1,...,1)}_\text{$N_{k}$ times}$.
\vspace{-1em}
\EndFor
\State $I(t) := \argmax_{k\in\{1,...,K\}}\{\left<{\bf R}_{k},{\bf W}_k\right>\}$
\State Pull arm $I(t)$ and observe reward $R_{t,I(t)}$.
\State Update history ${\bf R}_{I(t)} := ({\bf R}_{I(t)}^{\top} , R_{t,I(t)})^{\top}$ and count $N_{I(t)} := N_{I(t)} + 1$
\EndFor
\end{algorithmic}
\end{algorithm}

\end{document}